\documentclass[10pt]{article}

\usepackage{times}

\usepackage{amsopn,amsmath,amsthm,amssymb}

\usepackage{graphicx} % more modern
\usepackage{subfigure} 

\usepackage{color}

\usepackage[numbers,sort]{natbib}

\usepackage{algcompatible}
\usepackage{algorithm}

\usepackage{tabu}

\usepackage{hyperref}

\usepackage{etoolbox}
\newtoggle{withappendix}
\newtoggle{arxiv}
\toggletrue{withappendix}
\togglefalse{arxiv}
\toggletrue{withappendix}\toggletrue{arxiv}

\iftoggle{arxiv}{
  \setlength{\textwidth}{6.5in}
  \setlength{\textheight}{9in}
  \setlength{\oddsidemargin}{0in}
  \setlength{\evensidemargin}{0in}
  \setlength{\topmargin}{-0.5in}
  \newlength{\defbaselineskip}
  \setlength{\defbaselineskip}{\baselineskip}
  \setlength{\marginparwidth}{0.8in}
}{
  \usepackage[accepted]{icml2016}
}

\usepackage{multibib}
\newcites{sec}{Secondary Literature}

\usepackage{breqn}

\usepackage{xspace}

\usepackage{tikz}
\usetikzlibrary{shapes,backgrounds}
\usetikzlibrary{arrows}

\usepackage{units}
\usepackage{lipsum}

\algblockdefx{FORP}{ENDFORP}[1]%
  {\textbf{for }#1 \textbf{do in parallel without locks}}%
  {\textbf{end for}}

\def\compactify{\itemsep=0pt \topsep=0pt \partopsep=0pt \parsep=0pt}

\theoremstyle{definition}
\newtheorem{definition}{Definition}

\theoremstyle{plain}

\newtheorem{proposition}{Proposition}

\usepackage{thmtools}
\usepackage{thm-restate}

\newcommand{\N}{\mathbb{N}}

\newcommand{\Abs}[1]{\left| #1 \right| }
\newcommand{\Prob}[2][]{\underset{#1}{\mathbf{P}}\left( \hiderel{#2} \right) }
\newcommand{\Probc}[2]{\mathbf{P}\left( #1 \middle | \hiderel{#2} \right) }
\newcommand{\Exv}[1]{\mathbf{E}\left[ #1 \right]}
\newcommand{\Exvc}[2]{\mathbf{E}\left[ #1 \middle | \hiderel{#2} \right]}

\newcommand{\F}{\mathcal{F}}

\newcommand{\tvdist}[1]{\left\| #1 \right\|_{\mathrm{TV}}}
\newcommand{\svdist}[2]{\left\| #1 \right\|_{\mathrm{SV}(#2)}}

\newcommand{\hogwild}{\textsc{Hogwild!}\xspace}

\newcommand{\crcchange}[1]{#1}

\iftoggle{arxiv}{
  \newlength{\myfcwidth}
  \setlength{\myfcwidth}{0.75\columnwidth}
}{
  \newlength{\myfcwidth}
  \setlength{\myfcwidth}{\columnwidth}
}

\begin{document}

\iftoggle{arxiv}{
  
\title{Ensuring Rapid Mixing and Low Bias for Asynchronous Gibbs Sampling}

\author{
Christopher De Sa,
Kunle Olukotun, and
Christopher R{\'e} \\
Stanford University, \\
\texttt{\{cdesa,kunle,chrismre\}@stanford.edu}
}

\maketitle

}{

\twocolumn[
\icmltitle{Ensuring Rapid Mixing and Low Bias for Asynchronous Gibbs Sampling}

\icmlauthor{Christopher De Sa}{cdesa@stanford.edu}
\icmladdress{Department of Electrical Engineering,
  Stanford University, Stanford, CA 94309}
\icmlauthor{Kunle Olukotun}{kunle@stanford.edu}
\icmladdress{Department of Electrical Engineering,
  Stanford University, Stanford, CA 94309}
\icmlauthor{Christopher R{\'e}}{chrismre@stanford.edu}
\icmladdress{Department of Computer Science,
  Stanford University, Stanford, CA 94309}

\icmlkeywords{hogwild, gibbs sampling}

\vskip 0.3in
]

}

% \maketitle

\begin{abstract} 
Gibbs sampling is a Markov chain Monte Carlo technique commonly used for
estimating marginal distributions.
To speed up Gibbs sampling, there has recently been interest in parallelizing
it by executing asynchronously.  While empirical results suggest
that many models can be efficiently sampled asynchronously, traditional
Markov chain analysis does not apply to the asynchronous case, and thus
asynchronous Gibbs sampling is poorly understood.
In this paper, we derive a better understanding of the two main challenges
of asynchronous Gibbs: bias and mixing time. 
We show experimentally that our theoretical results match practical outcomes.
\end{abstract}

\section{Introduction}

Gibbs sampling is one of the most common Markov chain Monte Carlo methods used
with graphical models~\cite{koller2009probabilistic}.  
In this setting, Gibbs sampling (Algorithm
\ref{algGibbsSampling}) operates iteratively by choosing at random a variable
from the model at each timestep, and updating it by sampling from its
conditional distribution given the other variables in the model.
Often, it 
is applied to inference problems, in which we are trying to estimate the
marginal probabilities of some query events in a given distribution.  

\begin{algorithm}[h]
  \caption{Gibbs sampling}
  \label{algGibbsSampling}
  \begin{algorithmic}
    \REQUIRE Variables $x_i$ for $1 \le i \le n$, and distribution $\pi$.
    \FOR{$t = 1$ \textbf{to} $T$}
      \STATE Sample $s$ uniformly from $\{1, \ldots, n \}$.
      \STATE Re-sample $x_s$ uniformly from $\mathbf{P}_{\pi}(X_s | X_{\{1,\ldots,n\} \setminus \{ s \}})$.
    \ENDFOR
  \end{algorithmic}
\end{algorithm}

For sparse graphical models, to which
Gibbs sampling is often applied, each of these updates needs to read the values
of only a small subset of the variables; therefore each update can be computed
very quickly on modern hardware.  Because of this and other useful properties
of Gibbs sampling, many systems use Gibbs sampling to perform inference on
big data \cite{newman2007distributed,lunn2009bugs,
mccallum2009factorie,smola2010architecture,NIPS2012_4832,zhang2014dimmwitted}.

Since Gibbs sampling is such a ubiquitous algorithm, it is important to
try to optimize its execution speed on modern hardware.
Unfortunately, while modern computer hardware has been trending towards more
parallel architectures~\cite{sutter2005lunch}, traditional Gibbs sampling is an
inherently sequential
algorithm; that is, the loop in Algorithm \ref{algGibbsSampling} is not
\crcchange{directly parallelizable}.
Furthermore, for sparse models, very little work
happens within each iteration, meaning it is difficult to extract much
parallelism from the body of this loop.  Since traditional Gibbs sampling
parallelizes so poorly, it is interesting to study variants of
Gibbs sampling that can be parallelized.  Several such variants have been
proposed, including applications to latent Dirichlet allocation
\cite{newman2007distributed,smola2010architecture} and distributed
constraint optimization problems
\cite{nguyen2013distributed}.

In one popular variant, multiple threads run the Gibbs sampling update rule
in parallel without locks, a strategy called \emph{asynchronous}
or \hogwild execution---in this paper, we use these two terms interchangeably.
This idea was proposed, but not analyzed theoretically,
in \citet{smola2010architecture},
and has been shown to give empirically better
results on many models~\cite{zhang2014dimmwitted}.  But when can
we be sure that \hogwild
Gibbs sampling will produce accurate results?  Except for the case of
Gaussian random
variables~\cite{johnson2013analyzing}, there is no existing analysis by
which we can ensure that asynchronous Gibbs sampling will be appropriate for a
particular application.  Even the problems posed by \hogwild-Gibbs are poorly
understood, and their solutions more so.

% \begin{algorithm}[h]
%   \caption{Asynchronous (\hogwild) Gibbs sampling}
%   \label{algHogwildGibbsSampling}
%   \begin{algorithmic}
%     \REQUIRE Variables $x_i$ for $1 \le i \le n$, and distribution $\pi$.
%     \FORP{$i = 1$ \textbf{to} $\mathtt{numthreads}$}
%       \FOR{$t = 1$ \textbf{to} $T / \mathtt{numthreads}$}
%         \STATE Sample $s$ uniformly from $\{1, \ldots, n \}$.
%         \STATE Sample $x_s$ uniformly from $\mathbf{P}_{\pi}(X_s | X_{\{1,\ldots,n\} \setminus \{ s \}})$.
%       \ENDFOR
%     \ENDFORP
%   \end{algorithmic}
% \end{algorithm}

As we will \crcchange{show} in the following sections, there are two main issues
when analyzing asynchronous Gibbs sampling.  Firstly, we will show by example
that, surprisingly, \hogwild-Gibbs can be \emph{biased}---unlike sequential
Gibbs, it does not always produce samples that are arbitrarily close to the
target distribution.  Secondly, we will show that the
\emph{mixing time} (the time for the chain to become close to its stationary
distribution) of asynchronous Gibbs sampling can be up to exponentially
greater than that of the corresponding sequential chain.

To address the issue of bias, we need some way to describe the distance between
the target distribution $\pi$ and the distribution of the samples produced
by \hogwild-Gibbs.  The standard notion to use here is the \emph{total
variation distance}, but for the task of computing marginal probabilities,
it gives an overestimate on the error caused by bias.  To better describe the
bias, we introduce a new notion of statistical distance, the
\emph{sparse variation distance}.  While this relaxed notion of
statistical distance is interesting in its own right, its main benefit
here \crcchange{is that it uses a more local view of the chain to more
tightly measure the effect of bias.}

Our main goal is to identify conditions under which the bias and mixing
time of asynchronous Gibbs can be bounded.
One parameter that has been used to great effect in the 
analysis of Gibbs sampling is the \emph{total influence} $\alpha$ of a
model.  The total influence measures the degree to which the marginal
distribution of a variable can depend on the values of the other variables in
the model---this parameter has appeared as part of a celebrated
line of work on \emph{Dobrushin's condition} ($\alpha < 1$), which ensures the
rapid mixing of spin statistics
systems~\cite{dobrushin1956central,dyer06dobrushin,hayes2006simple}.
It turns out that we can use this
parameter to bound both the bias and mixing time of \hogwild-Gibbs,
and so we make the following contributions:
\begin{itemize}
  \compactify
  \item We describe a way to statistically model the asynchronicity in
    \hogwild-Gibbs sampling.
  \item To bound the bias, we prove that for classes of models with
    bounded total influence $\alpha = O(1)$,
    if sequential Gibbs sampling achieves small sparse variation distance to
    $\pi$ in $O(n)$ steps,
    where $n$ is the number of variables,
    then \hogwild-Gibbs samples achieve the same distance in
    at most $O(1)$ more steps.
  \item For models that satisfy Dobrushin's condition (that is,
    $\alpha < 1$), we show that the mixing time bounds
    of sequential and \hogwild-Gibbs sampling differ only by a factor
    of $1 + O(n^{-1})$.
  \item We validate our results experimentally and show that, by using
    asynchronous execution, we can achieve wall-clock speedups of up to
    $2.8 \times$ on real problems.
\end{itemize}

\section{Related Work}

Much work has been done on the analysis of parallel Gibbs samplers.  One
simple way to parallelize Gibbs sampling is to run multiple chains independently
in parallel: this heuristic uses parallelism to produce more samples overall,
but does not produce accurate samples more quickly.  Additionally, this strategy
is sometimes worse than other strategies on a systems
level~\cite{smola2010architecture,zhang2014dimmwitted}, typically because it
requires additional memory
to maintain multiple models of the chain.
Another strategy for parallelizing Gibbs sampling involves taking advantage
of the structure of the underlying factor graph to run in parallel while still
maintaining an execution pattern to which the standard sequential Gibbs sampling
analysis can be applied~\cite{gonzalez2011parallel}.  Much further work has
focused on parallelizing sampling for specific problems, such as
LDA~\cite{newman2007distributed,smola2010architecture} and
others~\cite{nguyen2013distributed}.

Our approach follows on the paper of \citet{johnson2013analyzing}, which
\crcchange{named}
the \hogwild-Gibbs sampling algorithm and analyzed it for Gaussian
models.  Their main contribution is an analysis framework that
includes a sufficient condition under which
\hogwild Gaussian Gibbs samples are guaranteed to have the correct asymptotic
mean.  Recent work~\cite{terenin2015asynchronous} has analyzed a similar
algorithm under even stronger regularity conditions.  Here, we seek to give
more general results for the analysis of \hogwild-Gibbs sampling on
discrete-valued factor graphs.

The \hogwild-Gibbs sampling algorithm was inspired by a line of work on 
parallelizing stochastic gradient descent (SGD) by running it asynchronously.
\hogwild SGD was first proposed by \citet{recht2011hogwild}, who proved that
while running without locks causes race conditions, they do not significantly
impede the
convergence of the algorithm. The asynchronous execution strategy has been
applied to many problems---such as PageRank approximations~\cite{FrogWild},
deep learning~\cite{DogWild}
and recommender systems~\cite{yu2012scalable}---so it is not surprising that it
has been proposed for use with Gibbs sampling.  Our goal in this paper is to
combine
analysis ideas that have been applied to Gibbs sampling and \hogwild, in
order to characterize the behavior of asynchronous Gibbs.
In particular, we are motivated by some recent work on the analysis of
\hogwild for
SGD \cite{liu2013asynchronous,desa2015tamingthewild,mania2015perturbed,liu2015asynchronous}.
Several of these results suggest modeling the race conditions
inherent
in \hogwild SGD as noise in a stochastic process; this lets them bring a
trove of statistical techniques to bear on the analysis of \hogwild SGD.  
Therefore, in this paper, we will apply a similar stochastic process model to
Gibbs sampling.

Several recent papers have focused on the mixing time of Gibbs sampling
\crcchange{based on the structural properties of the model}.
\citet{gotovos2015submodular} and \citet{desa2015gibbs} each
show that Gibbs sampling mixes in polynomial time for a class of
distributions bounded by some parameter.  Unfortunately, these results both
depend on \emph{spectral methods} (that try to bound the spectral gap of the
Markov transition matrix), which are difficult to apply to \hogwild Gibbs
sampling for two reasons.  First, spectral methods don't let us represent
the sampler as a stochastic process, which limits the range of techniques we
can use to model the noise.  Secondly, while most spectral methods only apply
to \emph{reversible} Markov chains---and sequential Gibbs sampling is always
a reversible chain---for \hogwild-Gibbs sampling the asynchronicity and
parallelism make the chain non-reversible.  Because of this, we were unable to
use these spectral results in our asynchronous setting. We are forced to rely on
the other method~\cite{guruswami2000rapidly} for analyzing Markov processes,
\emph{coupling}---the type of analysis used with the Dobrushin
condition---which we will describe in the following sections.

% we were unable to
% leverage these results for analyzing \hogwild-Gibbs sampling, because they  The reason for
% this is that there are two entirely different classes of approach to proving
% rapid mixing for Markov chains~\cite{guruswami2000rapidly}.  Spectral methods,
% techniques that try to bound the spectral gap of the Markov transition matrix,
% are the most commonly used, including in the two papers mentioned above.
% Unfortunately, it is difficult to apply spectral methods to analyze \hogwild
% Gibbs sampling for two reasons. 

\section{Modeling Asynchronicity}
\label{ssModelingHogwild}

In this section, we describe a statistical model for asynchronous Gibbs
sampling by adapting the hardware model outlined in
\citet{desa2015tamingthewild}.
Because we are motivated by the factor graph inference problem, we
will focus on the case where the distribution $\pi$ that we want to sample
comes from a sparse,
discrete graphical model.

Any \hogwild-Gibbs implementation involves some number of threads each
repeatedly executing the Gibbs update rule on a single copy of the model
(typically stored in RAM).  We assume
that this model serializes all writes, such that we can speak of the state
of the system after $t$ writes have occurred.  We call this time $t$, and we
will model the \hogwild system as a stochastic process adapted to the natural
filtration $\F_t$.  \crcchange{Here,} 
$\F_t$ contains all events that have occurred up to
time $t$, and we say an event is \emph{$\F_t$ measurable} if it is known
deterministically by time $t$.

We begin our construction by letting
$x_{i,t}$ denote the ($\F_t$ measurable) value of variable $i$ at time $t$,
and letting $\tilde I_t$ be the ($\F_{t+1}$ measurable) index of the variable
that we choose to sample at time $t$.  For Gibbs sampling, we have
\[
  \forall i \in \{1, \ldots, n \}, \;
  \Probc{\tilde I_t = i}{\F_t} = \frac{1}{n};
\]
this represents the fact that we have an equal probability of sampling each
variable.

Now that we have defined which variables are to be sampled, we
proceed to describe how they are sampled.  For \hogwild-Gibbs sampling, we must
model the fact that the sampler does not get to use exactly the values of
$x_{i, t}$; rather
it has access to a cache containing potentially \emph{stale} values.  To do
this, we define ($\F_{t+1}$ measurable)
$\tilde v_{i,t} = x_{i,t-\tilde \tau_{i,t}}$,
where $\tilde \tau_{i,t} \ge 0$ is a  \emph{delay
parameter} ($\F_{t+1}$ measurable and independent of
$\tilde I_t$) that represents how old the currently-cached value for variable $i$
could be.  A variable resampled using this stale data would have distribution
\[
  \Probc{\tilde z_{i,t} = z}{\F_t}
  \propto
  \pi(\tilde v_{1,t}, \ldots, \tilde v_{i-1,t}, z,
    \tilde v_{i+1,t}, \ldots, \tilde v_{n,t}).
\]
Using this, we can relate
the values of the variables across time with
\[
  x_{i,t+1}
  = 
  \left\{
    \begin{array}{l l}
      \tilde z_{i,t} & \text{ if } i = \tilde I_t \\
      x_{i,t} & \text{ otherwise.}
    \end{array}
  \right.
\]

So far, our model is incompletely specified, because we have not described
the distribution of the delays $\tilde \tau_{i,t}$. Unfortunately, since these
delays \crcchange{depend on the number of threads and the specifics of the
hardware}~\cite{recht2011hogwild}, their distribution is
difficult to measure.
Instead of specifying a particular distribution, we require
only a bound on the expected delay,
$\Exvc{\tilde \tau_{i, t}}{\F_t} \le \tau$.
In this model, the
$\tau$ parameter represents everything that is relevant about the hardware;
representing the hardware in this way
has been successful for the analysis of asynchronous
SGD~\cite{recht2011hogwild}, so it is reasonable to
use it for Gibbs sampling. In addition to this, we will need a similar
parameter that 
bounds the tails of $\tilde \tau_{i,t}$ slightly more aggressively.  We require
that for some parameter $\tau^*$, and for all $i$ and $t$,
\[
  \Exvc{\exp\left( n^{-1} \tilde \tau_{i, t} \right) }{\F_t}
  \le
  1 + n^{-1} \tau^*.
\]
This parameter is typically very close to the expected value bound $\tau$; in
particular, as $n$ approaches infinity, $\tau^*$ approaches $\tau$.

\section{The First Challenge: Bias}

Perhaps the most basic result about sequential Gibbs sampling is the fact that,
in the limit of large numbers of samples, it is unbiased.  In order to measure
convergence of Markov chains to their stationary distribution, it is standard
to use the total variation distance.

\begin{definition}[Total Variation Distance]
\label{defnTotalVariationDistance}
The \emph{total variation distance}~\citep[p. 48]{levin2009markov} between two
probability measures $\mu$ and $\nu$ on probability space
$\Omega$ is defined as
\[
  \tvdist{\mu - \nu} = \max_{A \subset \Omega} \Abs{\mu(A) - \nu(A)},
\]
that is, the maximum difference between the probabilities that $\mu$ and
$\nu$ assign to a single event $A$.
\end{definition}

It is a well-known result that, for Gibbs sampling on a strictly-positive
target distribution $\pi$, it will hold that
\begin{equation}
  \label{eqnSequentialUnbiasedLimit}
  \lim_{t \rightarrow \infty}
  \tvdist{P^{(t)} \mu_0 - \pi}
  =
  0,
\end{equation}
where \crcchange{$P^{(t)} \mu_0$ denotes} the distribution of the $t$-th sample.

One of the difficulties that arises when applying \hogwild to Gibbs sampling is
that the race conditions from the asynchronous execution add bias to the
samples --- Equation \ref{eqnSequentialUnbiasedLimit} no longer holds.  To
understand why, we can consider a simple example.

\subsection{Bias Example}

Consider a simple model with two variables $X_1$ and $X_2$ each taking on 
values in $\{0,1\}$, and having distribution
\[
  p(0, 1) = p(1, 0) = p(1, 1) = \frac{1}{3}
  \qquad
  p(0, 0) = 0. 
\]
% Gibbs sampling on this model has the transition graph illustrated in Figure
% \ref{figSimpleSeqTrans};
\iftoggle{arxiv}{
\begin{figure}[h]%
\centering
\resizebox{.65\myfcwidth}{!}{%
\begin{tikzpicture}[every node/.style={inner sep=0,outer sep=0}]
\draw (0,0) node[draw,circle,minimum size=0.8cm] (a00) {\small $(0,0)$};
\draw (0,2) node[draw,circle,minimum size=0.8cm] (a01) {\small $(0,1)$};
\draw (2,0) node[draw,circle,minimum size=0.8cm] (a10) {\small $(1,0)$};
\draw (2,2) node[draw,circle,minimum size=0.8cm] (a11) {\small $(1,1)$};
\path (a11) edge[->,bend right] node [above=0.05cm] {\footnotesize $\nicefrac{1}{4}$} (a01);
\path (a01) edge[->,bend right] node [below=0.05cm] {\footnotesize $\nicefrac{1}{4}$} (a11);
\path (a10) edge[->,bend right] node [right=0.05cm] {\footnotesize $\nicefrac{1}{4}$} (a11);
\path (a11) edge[->,bend right] node [left=0.05cm] {\footnotesize $\nicefrac{1}{4}$} (a10);
\path (a01) edge[->,loop left] node [left=0.05cm] {\footnotesize $\nicefrac{3}{4}$} (a01);
\path (a10) edge[->,loop right] node [right=0.05cm] {\footnotesize $\nicefrac{3}{4}$} (a10);
\path (a11) edge[->,loop right] node [right=0.05cm] {\footnotesize $\nicefrac{1}{2}$} (a11);
\path (a00) edge[->] node [left=0.05cm] {\footnotesize $\nicefrac{1}{2}$} (a01);
\path (a00) edge[->] node [below=0.05cm] {\footnotesize $\nicefrac{1}{2}$} (a10);
\end{tikzpicture}}
\caption{
Transition graph of sequential Gibbs on example model.
}
\label{figSimpleSeqTrans}
\end{figure}
}{}

Sequential Gibbs sampling on this model will produce unbiased samples from the
target distribution.
Unfortunately, this is not the case if we run \hogwild-Gibbs sampling on this
model.  Assume that the state is currently $(1,1)$ and two threads, $T_1$ and
$T_2$, simultaneously update $X_1$ and $X_2$ respectively.  Since $T_1$ reads
state $(1,1)$ it will update $X_1$ to $0$ or $1$ each with probability
$0.5$; the same will be true for $T_2$ and $X_2$.  Therefore, after
this happens, every state will have probability $0.25$; this includes
the state $(0,0)$ which should never occur!  Over time, this race condition
will produce samples with value $(0,0)$ with some non-zero frequency; this
is an example of \emph{bias} introduced by the \hogwild sampling.  Worse,
this bias is not just theoretical: Figure \ref{figHogwildBias} illustrates
how the measured distribution for this model is affected by two-thread
asynchronous execution.
In particular, we observe that almost $5\%$ of the mass is erroneously measured
to be in the state $(0,0)$, which has no mass at all in the true distribution.
The total variation distance to the target distribution is quite large at
$9.8 \%$, and, unlike in the sequential case, this
bias doesn't disappear as the number of samples goes to infinity.

\begin{figure}[t]%
\centering
\resizebox{!}{.60\myfcwidth}{%
\large \input{plotbias.tex}%
}%
\caption{Bias introduced by \hogwild-Gibbs ($10^6$ samples).}%
\label{figHogwildBias}%
\end{figure}

% Note that this sort of bias from race conditions also occurs in \hogwild
% SGD~\cite{recht2011hogwild}; however, for SGD, the bias can be driven to
% zero by decreasing the step size.  Unfortunately, no such parameter exists
% for Gibbs sampling, so some level of bias is inevitable in \hogwild-Gibbs
% sampling.

\subsection{Bounding the Bias}
\label{secSparseVariationDistance}

The previous example has shown that asynchronous Gibbs sampling will not
necessarily
produce a sequence of samples arbitrarily close to the target distribution.
Instead, the samples \crcchange{may}
approach some other distribution, which we hope is
sufficiently similar for some practical purpose.  Often, the purpose of
Gibbs sampling is to estimate the marginal distributions of individual variables
or of events that each depend on only a small number of variables in the model.
To characterize the
accuracy of these estimates, the total variation distance is \emph{too
conservative}: it depends on the difference over all the events in the space,
when most of these are events that we do not care about.  To address this, we
introduce the following definition.

\begin{definition}[Sparse Variation Distance]
  \label{defnSparseVariationDistance}
  For any event $A$ in a probability space $\Omega$ over
  a set of variables $V$, let $\Abs{A}$ denote
  the number of variables upon which $A$ depends.
  Then, for any two distributions $\mu$ and $\nu$ over $\Omega$, we define
  the \emph{$\omega$-sparse variation distance} to be
  \[
    \svdist{\mu - \nu}{\omega}
    =
    \max_{\Abs{A} \le \omega}
    \Abs{\mu(A) - \nu(A)}.
  \]
\end{definition}
For the wide variety of applications that use sampling for marginal estimation,
the sparse variation distance measures the quantity
we actually care about: the maximum possible bias in the marginal distribution
of the samples.  As we will show, asynchronous execution seems to have less
effect
on the sparse variation distance than the total variation distance,
\crcchange{because sparse variation distance uses a more localized
view of the chain.}
For example, in Figure \ref{figHogwildBias}, the total variation distance
between the sequential and \hogwild distributions is $9.8 \%$, while
the $1$-sparse variation
distance is only $0.4 \%$. That is, while \hogwild execution does
introduce great bias into the distribution, it still estimates marginals of
the individual variables accurately.

This definition suggests the question: how long do we have to run before our
samples have low sparse variation distance from the target distribution?  To
answer this question, we introduce the following definition.
\begin{definition}[Sparse Estimation Time]
  The \emph{$\omega$-sparse estimation time} of a stochastic sampler
  with distribution $P^{(t)} \mu_0$ at time $t$ and target distribution
  $\pi$ is the
  first time $t$ at which, for any initial distribution $\mu_0$, the estimated
  distribution is within sparse variation distance
  $\epsilon$ of $\pi$,
  \[
    t_{\mathrm{SE}(\omega)}(\epsilon)
    =
    \min \{
      t \in \N
    \mid
      \forall \mu_0, \,
      \| P^{(t)} \mu_0 - \pi \|_{\mathrm{SV}(\omega)}
      \le
      \epsilon
    \}.
  \]
\end{definition}

% Note that, in comparison to the definition of the mixing time, for \hogwild
% Gibbs sampling this definition compares the actual distribution we are
% trying to sample, rather than the biased stationary distribution of the
% \hogwild chain.
% It is easy to show that the sparse estimation time of a Markov chain is always less
% than the mixing time; it is also straightforward to bound the sparse mixing
% time of the low-local-influence chains analyzed in the previous section.

In many practical systems~\cite{neubig14pgibbs,wu2015incremental}, Gibbs
sampling is used
without a proof that it works; instead, it is naively run for some fixed number
of passes through the dataset.  This naive strategy works for models for which
accurate marginal estimates can be achieved after $O(n)$ samples.
This $O(n)$ runtime is
necessary for Gibbs sampling to be feasible on big data, meaning
roughly that these are the models which it is interesting to try to
speed up using asynchronous execution.  Therefore, for the rest of this
section, we will focus on the bias of the \hogwild chain for this class of
models.
% To analyze \hogwild SGD, \citet{desa2015tamingthewild} propose using a
% \emph{martingale-based technique} that modifies an existing proof of
% convergence for a sequential stochastic algorithm into one that applies to its
% \hogwild variant.  
When analyzing Gibbs sampling, we can bound the bias
%their ideas can be applied
within the context of a coupling argument using a parameter called the
\emph{total influence}.  While we arrived at this condition independently,
it has been studied before, especially in the context of \emph{Dobrushin's
condition}, which ensures rapid mixing of Gibbs sampling. 

\begin{definition}[Total Influence]
  \label{defnLocalInfluence}
  Let $\pi$ be a probability distribution over some set of variables $I$.
  Let $B_j$ be the
  set of state pairs $(X, Y)$ which differ only at variable $j$.  Let
  $\pi_i(\cdot | X_{I \setminus \{i\}})$ denote the conditional distribution
  in $\pi$ of variable
  $i$ given all the other variables in state $X$.  Then, define $\alpha$, the
  total influence of $\pi$, as
  \[
    \alpha
    =
    \max_{i \in I}
    \sum_{j \in I}
    \max_{(X, Y) \in B_j}
    \tvdist{
      \pi_i(\cdot | X_{I \setminus \{i\}})
      -
      \pi_i(\cdot | Y_{I \setminus \{i\}})
    }.
  \]
  We say the model satisfies Dobrushin's condition if $\alpha < 1$.
\end{definition}

% While we arrived at this condition independently, it has been studied before
% as part of \emph{Dobrushin's condition} for ensuring rapid mixing of spin-statistics
% models~\cite{dobrushin1956central,dyer06dobrushin,hayes2006simple}.
One way to think of total influence for factor graphs is as a generalization of
maximum degree; indeed, if a factor graph has maximum degree $\Delta$, it can
easily be shown that $\alpha \le \Delta$.
It turns out that if we can bound both this parameter and
the sparse estimation time
of sequential Gibbs sampling, we can give a simple bound on the sparse
estimation time for asynchronous Gibbs sampling.

\begin{restatable}{rclaim}{claimHogwildGeneralBiasBigO}
  \label{claimHogwildGeneralBiasBigO}
  Assume that we have a class of distributions with bounded total influence
  $\alpha = O(1)$.  For each distribution $\pi$ in the class, let
  $\bar t_{\mathrm{SE-seq}(\omega)}(\pi, \epsilon)$ be an upper bound on the
  $\omega$-sparse estimation time of its sequential Gibbs sampler, and
  assume that it is a convex, decreasing function of $\epsilon$.
  Further assume that, for any $\epsilon$, across all models,
  \[
    \bar t_{\mathrm{SE-seq}(\omega)}(\pi, \epsilon) = O(n),
  \]
  where $n$ is the number of variables in the model.
  Then, for any $\epsilon$, the sparse estimation time of \hogwild-Gibbs 
  across all models is bounded by
  \[
    t_{\mathrm{SE-hog}(\omega)}(\pi, \epsilon)
    \le
    \bar t_{\mathrm{SE-seq}(\omega)}(\pi, \epsilon)
    +
    O(1).
  \]
\end{restatable}

% That is, for any model of bounded total influence for
% which only $O(n)$ samples are sufficient for good marginal estimation,
% \hogwild-Gibbs sampling will produce equivalent results to sequential Gibbs
% sampling in only a constant number of additional steps.
Roughly, this means that \hogwild-Gibbs sampling ``works'' on all problems 
for which we know marginal estimation is ``fast'' and the total influence is
bounded.
Since the sparse estimation times here are measured in iterations, and the
asynchronous sampler is able, due to parallelism, to run many more iterations
in the same amount of wall clock time, this result implies that \hogwild-Gibbs
can be much faster than sequential Gibbs for producing estimates of
similar quality.  
% While this
% corollary tells us that any error $\epsilon$ can be achieved for sufficiently
% large problem sizes, it does not actually tell us what error can be achieved
% for any particular $n$.
\crcchange{To prove Claim \ref{claimHogwildGeneralBiasBigO},
and more explicitly bound the bias, we use
the following lemma.}

\begin{restatable}{rlemma}{lemmaHogwildBiasRFive}
  \label{lemmaHogwildBiasRFive}
  \crcchange{
  Assume that we run \hogwild-Gibbs sampling on a distribution $\pi$ with
  total influence $\alpha$.  Let $P^{(t)}_{\mathrm{hog}}$
  denote the transition matrix of \hogwild-Gibbs and $P^{(t)}_{\mathrm{seq}}$
  denote the transition matrix of sequential Gibbs.  Then for any initial
  distribution $\mu_0$ and for any $t$,
  \begin{dmath*}
    \svdist{P^{(t)}_{\mathrm{hog}} \mu_0 - P^{(t)}_{\mathrm{seq}} \mu_0}{\omega}
    \le
    \frac{\omega \alpha \tau t}{n^2}
    \exp\left( \frac{(\alpha - 1)_+}{n} t \right),
  \end{dmath*}
  where $(x)_+$ denotes $x$ if $x > 0$ and $0$ otherwise.}
\end{restatable}

\crcchange{This lemma bounds the distance between the distributions of
asynchronous and
sequential Gibbs; if we let $t$ be the sparse estimation time of sequential
Gibbs, we can interpret this distance as an upper bound on the bias. 
% The required bound on $\epsilon$ in Theorem \ref{thmHogwildGeneralBias} is to
% be expected:
% because of bias, it is impossible to produce arbitrarily good samples with
% \hogwild-Gibbs.  Looked at another way, this bound on $\epsilon$ 
% characterizes the inherent bias of \hogwild sampling, in the limit as
% $t$ becomes large.
When $t = O(n)$, this bias is $O(n^{-1})$,
which has an intuitive explanation: for \hogwild execution, race conditions
occur about once every $\Theta(n)$ iterations, so the bias is roughly
proportional to the frequency of race conditions.  This gives us a
relationship between the statistical error of the algorithm and a more
traditional notion of computational error.}

Up until now, we have been assuming that we have a class for which the
sparse estimation time is $O(n)$. Using the total influence $\alpha$, we can
identify a class of models known to meet this criterion.

\begin{restatable}{rtheorem}{thmSequentialLocalBias}
  \label{thmSequentialLocalBias}
  For any distribution that 
  satisfies Dobrushin's condition, $\alpha < 1$, the $\omega$-sparse
  estimation time of the sequential Gibbs sampling process will be bounded by
  \[
    t_{\mathrm{SE-seq}(\omega)}(\epsilon)
    \le
    \left\lceil
      \frac{n}{1 - \alpha} \log\left( \frac{\omega}{\epsilon} \right)
    \right\rceil.
  \]
\end{restatable}

This surprising result says that, in order to produce good marginal
estimates for any model that satisfies Dobrushin's condition, we need only
$O(n)$ samples!  While we could
\crcchange{now use Lemma \ref{lemmaHogwildBiasRFive}}
to bound the sparse estimation time for \hogwild-Gibbs, a more direct
analysis produces a slightly better result, which we present here.

\begin{restatable}{rtheorem}{thmHogwildLocalBias}
  \label{thmHogwildLocalBias}
  For any distribution that satisfies Dobrushin's condition, $\alpha < 1$,
  and for any $\epsilon$ that satisfies
  \[
    \epsilon \ge 2 \omega \alpha \tau (1 - \alpha)^{-1} n^{-1},
  \]
  the $\omega$-sparse estimation time of the \hogwild Gibbs sampling process
  will be bounded by
  \[
    t_{\mathrm{SE-hog}(\omega)}(\epsilon)
    \le
    \left\lceil
    \frac{n}{1 - \alpha}
    \log\left( \frac{\omega}{\epsilon} \right)
    +
    \frac{2 \omega \alpha \tau}{(1 - \alpha)^2 \epsilon}
    \right\rceil.
  \]
\end{restatable}

This result gives us a definite class of models for which \hogwild-Gibbs
sampling is guaranteed to produce accurate marginal estimates quickly.

\section{The Second Challenge: Mixing Times}
\label{ssHogwildMixingTimes}

Even though the \hogwild-Gibbs sampler produces biased estimates, it is still
interesting to analyze how long we need to run it
before the samples it produces are independent of its initial conditions.
To measure the efficiency of a Markov chain, it is standard to
use the
\emph{mixing time}.
%, the time necessary to produce samples that are close, in
%terms of total variation distance, to the stationary distribution.

\begin{definition}[Mixing Time]
The \emph{mixing time}~\citep[p. 55]{levin2009markov} of a stochastic process
\crcchange{with transition matrix $P^{(t)}$ at time $t$}
and target distribution $\pi$ is the
first time $t$ at which, for any initial distribution $\mu_0$, the estimated
distribution is within TV-distance
\crcchange{$\epsilon$ of $P^{(t)} \pi$}.  That is,
\[
  t_{\mathrm{mix}}(\epsilon)
  =
  \min \left\{
    t
  \middle|
    \forall \mu_0, \,
    \tvdist{P^{(t)} \mu_0 - P^{(t)} \pi}
    \le
    \epsilon
  \right\}.
\]
\end{definition}

\subsection{Mixing Time Example}

As we did with bias, here we construct an example model for which asynchronous
execution
disastrously increases the mixing time.  The model we will construct
is rather extreme; we choose this model because simpler, practical models do
not seem to exhibit this type of catastrophic increase in the mixing time.
We start, for some odd constant $N$, with $N$ variables
$X_1, \ldots, X_N$ all in $\{ -1, 1 \}$,
and one factor with energy
\[
  \phi_X(X)
  =
  -M_1 \Abs{\mathbf{1}^T X},
\]
for some very large energy parameter $M_1$.
The resulting distribution will be almost uniform over all states with
$\mathbf{1}^T X \in \{ -1, 1 \}$.
To this model, we add another bank of variables $Y_1, \ldots, Y_N$ all in
$\{ -1, 1 \}$.  These variables also have a single associated factor with
energy
\[
  \phi_Y(X, Y)
  =
  \left\{
    \begin{array}{l l}
      \frac{\beta}{N} \left( \mathbf{1}^T Y \right)^2
        & \text{ if } \Abs{\mathbf{1}^T X} = 1 \\
      M_2 \left( \mathbf{1}^T Y \right)^2
        & \text{ if } \Abs{\mathbf{1}^T X} > 1
    \end{array}
  \right., 
\]
for parameters $\beta$ and $M_2$.
Combining these two factors gives us the overall distribution for our model,
\[
  \pi(X,Y)
  =
  \frac{1}{Z} \exp\left( \phi_X(X) + \phi_Y(X, Y) \right),
\]
where $Z$ is the constant necessary for this to be a distribution.  Roughly,
the $X$ dynamics are constructed to regularly ``generate'' race conditions,
while
the $Y$ dynamics are chosen to ``detect'' these race conditions and mix very
slowly as a result.  This
model is illustrated in Figure \ref{figBadMixModel}.
% For ease of analysis, we assume that our large constant $M_2$ satisfies
% $M_1 \gg n^2 M_2$,
% such that the sequential dynamics of the $X$ part of the chain are almost
% totally unaffected by the dynamics of the $Y$ part of the chain.

% TODO: Add to this section a short mention of the model that mixes very slowly
%       in sequential Gibbs but rapidly in asynchronous Gibbs.

\begin{figure}[h]%
\centering
\resizebox{.55\myfcwidth}{!}{\Large%
\begin{tikzpicture}[every node/.style={inner sep=0,outer sep=0}]
\draw (0,1.3) node[draw,fill=lightgray,rectangle,minimum size=1.0cm](px) {$\phi_X$};
\draw (-3,0) node[draw,circle,fill=white,minimum size=1.0cm](x1) {$X_1$};
\draw (-1.5,0) node[draw,circle,fill=white,minimum size=1.0cm](x2) {$X_2$};
\draw (0,0) node[draw,circle,fill=white,minimum size=1.0cm](x3) {$X_3$};
\draw (1.5,0) node(xcd){$\cdots$};
\draw (3,0) node[draw,circle,fill=white,minimum size=1.0cm](x4) {$X_N$};
\draw (0,-1.3) node[draw,fill=lightgray,rectangle,minimum size=1.0cm](py) {$\phi_Y$};
\draw (-3,-2.6) node[draw,circle,fill=white,minimum size=1.0cm](y1) {$Y_1$};
\draw (-1.5,-2.6) node[draw,circle,fill=white,minimum size=1.0cm](y2) {$Y_2$};
\draw (0,-2.6) node[draw,circle,fill=white,minimum size=1.0cm](y3) {$Y_3$};
\draw (1.5,-2.6) node(ycd){$\cdots$};
\draw (3,-2.6) node[draw,circle,fill=white,minimum size=1.0cm](y4) {$Y_N$};
\draw (px) -- (x1);
\draw (px) -- (x2);
\draw (px) -- (x3);
\draw (px) -- (xcd);
\draw (px) -- (x4);
\draw (py) -- (x1);
\draw (py) -- (x2);
\draw (py) -- (x3);
\draw (py) -- (xcd);
\draw (py) -- (x4);
\draw (py) -- (y1);
\draw (py) -- (y2);
\draw (py) -- (y3);
\draw (py) -- (ycd);
\draw (py) -- (y4);
\end{tikzpicture}}
\caption{
Factor graph model for mixing time example.
}
\label{figBadMixModel} 
\end{figure}

We simulated two-thread \hogwild-Gibbs on this model, measuring the marginal
probability
that $\mathbf{1}^T Y > 0$; by symmetry, this event has probability
$0.5$ in the stationary distribution for both the sequential and asynchronous
samplers.
Our results, for 
a model with $N = 2001$, $\beta = 0.3$, $M_1 = 10^{10}$, and $M_2 = 100$,
and initial state $X = Y = \mathbf{1}$, are plotted in Figure
\ref{figHogwildBadMix}.  Notice that, while the sequential sampler achieves the 
correct marginal probability relatively quickly, the asynchronous samplers
take a much longer time to achieve the correct result, even for a relatively
small expected delay ($\tau = 0.5$). 
% Worse, for the only-slightly-increased
% delay of $\tau = 2$, the event $\mathbf{1}^T Y < 0$ never even occurred once
% across all our experiments.
These results suggest that something catastrophic
is happening to the mixing time when we switch from sequential to asynchronous
execution --- and in fact we can prove this is the case.

\begin{figure}[t]%
\centering
\resizebox{!}{.60\myfcwidth}{%
\large \input{plotbadmix.tex}%
}%
\caption{Example wherein asynchronous sampling greatly increases
in mixing time. Marginals computed over $10^4$ trials.}%
\label{figHogwildBadMix}%
\end{figure}

\begin{restatable}{rstatement}{stmtBadMixExample}
  \label{stmtBadMixExample}
  For the example model described above, there exist parameters $M_1$, $M_2$,
  and $\beta$ (as a function of $N$) such that the mixing time of sequential
  Gibbs sampling is $O(N \log N)$ but the mixing time of \hogwild-Gibbs
  sampling, even with $\tau = O(1)$, can be $\exp(\Omega(N))$.
\end{restatable}

The intuition behind this statement is that for sequential Gibbs, the dynamics
of the $X$ part of the chain quickly causes it to have
$\Abs{\mathbf{1}^T X} = 1$, and then remain there for the remainder of the
simulation with high probability.  This in turn causes the energy of the
$\phi_Y$ factor to be essentially $\frac{\beta}{N} (\mathbf{1}^T Y)^2$, a
model which is known to be fast-mixing because it satisfies Dobrushin's
condition.  On the other hand, for \hogwild
Gibbs, due to race conditions we will see $\Abs{\mathbf{1}^T X} \ne 1$
with constant probability; this will cause the effective energy of the $\phi_Y$
factor to be dominated by the $M_2 (\mathbf{1}^T Y)^2$ term, a
model that is known to
take exponential time to mix.

\subsection{Bounding the Mixing Time}

This example shows that fast
mixing of the sequential sampler alone is not sufficient to guarantee
fast mixing
of the \hogwild chain. Consequently, we look for classes of models for
which
we can say something about the mixing time of both sequential and
\hogwild-Gibbs.  Dobrushin's condition is well known to imply rapid mixing of
sequential Gibbs, and it turns out that we can leverage it again here to bound
the mixing time of \hogwild-Gibbs.

\begin{restatable}{rtheorem}{thmSeqHogLocalMixing}
  \label{thmSeqHogLocalMixing}
  Assume that we run Gibbs sampling on a distribution that
  satisfies Dobrushin's condition, $\alpha < 1$.  Then the mixing time of
  sequential Gibbs will be bounded by
  \[
    t_{\mathrm{mix-seq}}(\epsilon)
    \le
    \frac{n}{1 - \alpha} \log\left( \frac{n}{\epsilon} \right).
  \]
  Under the same conditions, the mixing time of \hogwild-Gibbs will be
  bounded by
  \[
    t_{\mathrm{mix-hog}}(\epsilon)
    \le
    \frac{n + \alpha \tau^*}{1 - \alpha}
    \log\left( \frac{n}{\epsilon} \right).
  \]
\end{restatable}

% Remarkably, under the same condition (and using roughly the same proof
% structure), we can give a similar bound for \hogwild-Gibbs.
% \begin{theorem}
%   \label{thmHogwildLocalMixing}
%   Assume that we run \hogwild-Gibbs sampling on a distribution $\pi$ with
%   total influence $\alpha < 1$.  Assume further that the model of \hogwild we are
%   using has a stationary distribution.  Then, the
%   mixing time of this process will be bounded by
%   \[
%     t_{\mathrm{mix-hog}}(\epsilon)
%     \le
%     \frac{n + \alpha \tau^*}{1 - \alpha}
%     \log\left( \frac{n}{\epsilon} \right).
%   \]
% \end{theorem}

% Compared to the sequential result in Theorem \ref{thmSequentialLocalMixing},
% the bound on the mixing time of the \hogwild chain
% only increased by $\Theta(\log n)$, which is small compared to the
% $\Theta(n \log n)$ mixing time of the sequential sampler.  Multiplicatively, we
% we can write this comparison as
\crcchange{The above
example does not contradict this result since it does not satisfy Dobrushin's
condition; in fact its total influence is very large and scales with $n$.}
We can compare these two mixing time results as
\begin{equation}
  \label{eqnHogSeqMixingComparison}
  t_{\mathrm{mix-hog}}(\epsilon)
  \approx
  \left(
    1
    +
    \alpha \tau^* n^{-1}
  \right)
  t_{\mathrm{mix-seq}}(\epsilon);
\end{equation}
the bounds on the mixing times differ by a negligible factor of $1 + O(n^{-1})$.
This result shows
that, for problems that satisfy Dobrusin's condition, \hogwild-Gibbs sampling
mixes in about the same time as sequential Gibbs sampling, and is therefore
a practical choice for generating samples.

\subsection{A Positive Example: Ising Model}
\label{ssIsingModel}

To gain intuition here, we consider a simple example.  The Ising
model~\cite{ising1925beitrag} on a graph $G = (V, E)$ is a model
over probability space $\{ -1, 1 \}^V$, and has
distribution
\[
  p(\sigma)
  =
  \frac{1}{Z} \exp\bigg(
    \beta \sum_{(x, y) \in E} \sigma(x) \sigma(y)
    +
    \sum_{x \in V} B_x \sigma(x)
  \bigg),
\]
where $\beta$ is a parameter that is called the \emph{inverse temperature},
the $B_x$ are parameters that encode a \emph{prior} on the variables,
and $Z$ is the normalization constant necessary for this to be a
distribution.  For graphs of maximum degree $\Delta$ and sufficiently small
$\beta$, a bound on the mixing time of Gibbs sampling is known when
$\Delta \tanh \beta \le 1$.
% \begin{theorem}[Theorem 15.1 from ~\citet{levin2009markov}]
% \label{thmSequentialGibbsIsing}
% Consider Gibbs sampling for the Ising model on a graph with $n$ vertices and
% maximal degree $\Delta$.  As long as $\Delta \tanh \beta \le 1$, it has mixing
% time
% \[
%   t_{\mathrm{mix-seq}}(\epsilon)
%   \le
%   \frac{n}{1 - \Delta \tanh \beta}
%   \log\left(\frac{n}{\epsilon} \right).
% \]
% \end{theorem}
It turns out that the total influence of the Ising model can be bounded by
$\alpha \le \Delta \tanh \beta$, and so this condition
is simply another way of writing Dobrushin's condition.
We can therefore apply Theorem \ref{thmSeqHogLocalMixing} to bound the mixing
time of \hogwild-Gibbs with
\[
  t_{\mathrm{mix}}(\epsilon)
  \le
  \frac{n + \tau^* \Delta \tanh \beta}{1 - \Delta \tanh \beta}
  \log\left(\frac{n}{\epsilon} \right).
\]
This illustrates that the class of graphs we are considering includes some
common, well-studied models.

\subsection{Proof Outline}
\label{ssLocalProofOutline}

Here, we briefly describe the technique used to prove
Theorem \ref{thmSeqHogLocalMixing}; for ease of presentation, we focus on the
case where every variable takes on values in
$\{-1, 1\}$. We start by introducing the idea of a coupling-based
argument~\citep[p. 64]{levin2009markov}, \crcchange{which starts by}
constructing two copies of the same Markov chain,
$X$ and $\bar X$, starting from different states but running together in the
same probability space (i.e. using the same sources of randomness).
For analyzing \hogwild-Gibbs sampling, we share randomness by having both
chains sample the same variable at each iteration and sample it such that the
resulting values are maximally correlated---additionally both chains are
subject to the same \hogwild delays $\tilde \tau_{i,t}$.

At some random
time, called the \emph{coupling time} $T_{\mathrm{c}}$, the chains will
become equal---\crcchange{regardless} of their initial conditions.
Using this, we can
bound the mixing time with
\[
  t_{\mathrm{mix}}(\epsilon)
  \le
  \min \{
    t
  \mid
    \mathbf{P}(T_{\mathrm{c}} > t)
    \le
    \epsilon
  \}.
\]

In order to bound the probability that the
chains are not equal at a particular time $t$, we focus on the
quantity
\begin{equation}
  \label{eqnPhiMaxProbBar}
  \phi_t = \max_i \Prob{ X_{i,t} \ne \bar X_{i,t} }.
\end{equation}
Under the conditions of Theorem \ref{thmSeqHogLocalMixing}, we are able to
bound this using the total influence parameter.
% the maximum probability that a variable will differ between the two chains.
% For any variable $i$, since we sample the same variable in both chains,
% % the independence of $\tilde I_t$ and $\tilde \tau_{j,t}$
% % implies that
% \begin{align*}
%   &\Prob{X_{i,t+1} \ne Y_{i,t+1}}
%   =
%   \left( 1 - \frac{1}{n} \right) \Prob{X_{i,t} \ne Y_{i,t}} \\
%   &\hspace{2em} +
%   \frac{1}{n}
%   \Exv{
%     \tvdist{
%       \pi_i(\cdot | \tilde U_t)
%       -
%       \pi_i(\cdot | \tilde V_t)
%     }
%   }.
% \end{align*}
% We want to bound the second term here; it turns out we can do so using the
% total influence parameter $\alpha$ and a little algebra.
% % \begin{lemma}
% %   \label{lemmaLocalInfExvDist}
% %   If $\pi$ is a distribution with total influence
% %   $\alpha$, and $X$ and $Y$ are two random variables that take on values in the
% %   state space of $\pi$, then for any variable $i$
% %   \[
% %     \Exv{
% %       \tvdist{
% %         \pi_i(\cdot | X)
% %         -
% %         \pi_i(\cdot | Y)
% %       }
% %     }
% %     \le
% %     \alpha
% %     \max_j \Prob{X_j \ne Y_j}.
% %   \]
% % \end{lemma}
% Doing this and maximizing both sides over $i$ yields
% \[
%   \phi_{t+1}
%   \le
%   \left( 1 - \frac{1}{n} \right)
%   \phi_t
%   +
%   \frac{\alpha}{n}
%   \max_j
%   \sum_{k = 0}^{\infty}
%   \Prob{\tilde \tau_{j,t} = k}
%   \phi_{t-k}.
% \]
% Note that there is dependence on 
% the past probabilities because of the delayed reads.  
From here, we notice that by the union bound, 
$
\mathbf{P}(T_{\mathrm{c}} > t)
\le
n \phi_t.
$
Combining this with Equation \ref{eqnPhiMaxProbBar}
and reducing the subsequent expression lets us bound the mixing time, producing
the result of
Theorem \ref{thmSeqHogLocalMixing}.

\section{Experiments}

Now that we have derived a theoretical characterization of the behavior of
\hogwild-Gibbs sampling, we examine whether this characterization holds up
under experimental evaluation.  First, we examine the mixing time claims we
made in Section \ref{ssHogwildMixingTimes}.  Specifically, we want to check
whether increasing the expected delay parameter $\tau^*$ actually increases the
mixing time as predicted by
Equation \ref{eqnHogSeqMixingComparison}.  
% We also want to verify that nothing
% catastrophic (and not predicted by theory) happens when we run \hogwild-Gibbs.

To do this, we simulated \hogwild-Gibbs sampling running on a random synthetic
Ising model graph of order $n = 1000$, degree $\Delta = 3$, inverse
temperature $\beta = 0.2$, and prior weights $E_x = 0$. 
% From Equation
% \ref{eqnIsingModelLocalInfluence}, this model has total influence
% $\alpha \le 0.6$, and
This model has total influence $\alpha \le 0.6$, and
Theorem \ref{thmSeqHogLocalMixing} guarantees that it will mix rapidly.
Unfortunately, the mixing time of a chain is difficult to calculate
experimentally.  While techniques such as coupling from the
past~\cite{propp1996exact} exist for estimating the mixing time, using these
techniques in order to expose the (relatively small) dependence of the mixing
time on $\tau$ proved to be computationally intractable.

Instead, we use a technique called coupling to the future.  We initialize two
chains, $X$ and $Y$, by setting all the variables in $X_0$ to $1$ and all the
variables in $Y_0$ to $-1$.  We proceed by simulating a
coupling between the two chains, and return
the coupling time $T_{\mathrm{c}}$.  Our estimate
of the mixing time will then be $\hat t(\epsilon)$, where
$\mathbf{P}(T_{\mathrm{c}} \ge \hat t(\epsilon)) = \epsilon$.

\begin{figure}[t]%
\centering
\resizebox{!}{.65\myfcwidth}{%
\large \input{plottauest.tex}%
}%
\caption{Comparison of estimated mixing time and theory-predicted (by
Equation \ref{eqnHogSeqMixingComparison}) mixing time
as $\tau$ increases for a synthetic Ising model graph
($n = 1000$, $\Delta = 3$).}%
\label{figMixingTimeTau}%
\end{figure}

\begin{restatable}{rstatement}{stmtExperimentalUpperBound}
  \label{stmtExperimentalUpperBound}
  This experimental estimate is an upper bound for the
  mixing time.  That is,
  $
    \hat t(\epsilon)
    \ge
    t_{\mathrm{mix}}(\epsilon)
  $.
\end{restatable}

To estimate $\hat t(\epsilon)$, we ran $10000$ instances of the coupling
experiment, and returned the sample estimate of $\hat t(1/4)$.  To compare
across a range of $\tau^*$, we selected the $\tilde \tau_{i,t}$ to be
independent and identically distributed according to the maximum-entropy
distribution supported on $\{0, 1, \ldots, 200 \}$ consistent with a particular
assignment of $\tau^*$.  The resulting
estimates are plotted as the blue series in Figure \ref{figMixingTimeTau}.
The red line represents the mixing time that would be predicted by naively
applying Equation \ref{eqnHogSeqMixingComparison} using the estimate of the
sequential mixing time as a starting point --- we can see that it is a very
good match for the experimental results.  This experiment shows that, at least
for one archetypal model, our theory accurately characterizes the behavior of \hogwild
Gibbs sampling as the delay parameter $\tau^*$ is changed, and that using
\hogwild-Gibbs doesn't cause the model to catastrophically fail
to mix.

\begin{figure}[t]%
\centering
\resizebox{!}{.65\myfcwidth}{%
\large \input{plotscaling.tex}%
}%
\caption{Speedup of \hogwild and multi-model Gibbs sampling on large KBP dataset ($11$ GB).}%
\label{figHogwildSpeedup}%
\end{figure}

Of course, in order for \hogwild-Gibbs to be useful, it must also speed up the
execution of Gibbs sampling on some practical models.  It is already known
that this is the case, as these types of algorithms been widely
implemented in practice~\cite{smyth2009asynchronous,smola2010architecture}.
To further test this, we ran
\hogwild-Gibbs sampling on a real-world $11$ GB Knowledge Base Population
dataset (derived from the TAC-KBP challenge)
using a machine with a single-socket, 18-core Xeon E7-8890 CPU and
$1$ TB RAM.  As a
comparison, we also ran a
``multi-model'' Gibbs sampler: \crcchange{this consists of multiple threads
with a single execution of Gibbs sampling running independently in each
thread}.  This sampler will
produce the same number of samples as \hogwild-Gibbs, but
\crcchange{will require more memory to store multiple copies of the model.}
% the samples could be
% of lower quality since they are split up among several independent chains.

Figure \ref{figHogwildSpeedup} reports the speedup, in terms
of wall-clock time, achieved by \hogwild-Gibbs on this dataset.  
On this machine, we get speedups of up to $2.8 \times$, although the program
becomes memory-bandwidth bound at around $8$ threads, and we see no 
significant speedup beyond this.  With any number of workers, the run time of
\hogwild-Gibbs is close to that of multi-model Gibbs, which illustrates that
the additional cache contention caused by the \hogwild updates has little
effect on the algorithm's performance.
% \todo{get better speedups here}

\section{Conclusion}

We analyzed \hogwild-Gibbs sampling, a heuristic for parallelized MCMC
sampling, on discrete-valued graphical models.  First, we constructed a
statistical model for \hogwild-Gibbs by adapting a model already used for
the analysis of asynchronous SGD.  Next, we illustrated
a major issue with \hogwild-Gibbs sampling: that it produces biased
samples.  To address this, we proved that
if for some class of models with bounded total influence,
only $O(n)$ sequential Gibbs samples are necessary to produce good
marginal estimates,
%(in terms of sparse variation distance),
then \hogwild-Gibbs
sampling produces equally good estimates after only $O(1)$ additional steps.
Additionally, for models that satisfy Dobrushin's condition
($\alpha < 1$), we proved mixing time bounds for sequential and asynchronous
Gibbs sampling that differ by only a factor of $1 + O(n^{-1})$.
Finally, we showed that our theory matches experimental results, and that
\hogwild-Gibbs produces speedups up to $2.8 \times$ on a real dataset.

\subsection*{Acknowledgments} 

The authors acknowledge the support of:
DARPA FA8750-12-2-0335;
NSF IIS-1247701;
NSF CCF-1111943;
DOE 108845;
NSF CCF-1337375;
DARPA FA8750-13-2-0039; 
NSF IIS-1353606; ONR N000141210041
and N000141310129; NIH U54EB020405; Oracle; NVIDIA; Huawei; SAP Labs;
Sloan Research
Fellowship; Moore Foundation; American Family
Insurance; Google; and Toshiba.

``The views and conclusions contained herein are those of the authors and should not be interpreted as necessarily representing the official policies or endorsements, either expressed or implied, of DARPA, AFRL, NSF, ONR, NIH, or the U.S. Government.''

\bibliography{references}
\bibliographystyle{icml2016}

\iftoggle{withappendix}{

\clearpage

\onecolumn

\appendix

\section{Additional Bias Results}

In this section, we present the following additional result that bounds the
sparse estimation time of general Gibbs samplers.  In particular, this
theorem provides an explicit form of the result given in
Claim \ref{claimHogwildGeneralBiasBigO}.

\begin{restatable}{rtheorem}{thmHogwildGeneralBias}
  \label{thmHogwildGeneralBias}
  Assume that we run \hogwild-Gibbs sampling on a distribution $\pi$ with
  total influence $\alpha$.  Let $\bar t_{\mathrm{SE-seq}(\omega)}(\epsilon)$
  be some upper bound on the $\omega$-sparse estimation time of the
  corresponding sequential chain,
  and assume that it is a convex and decreasing function of $\epsilon$.
  For any $\epsilon > 0$, define
  \[
    c
    =
    \frac{1}{n}
    \bar t_{\mathrm{SE-seq}(\omega)}\left( \frac{\epsilon}{2} \right).
  \]
  Then, as long as $\epsilon$ is large enough that
  \[
    \epsilon
    \ge
    \frac{2 \omega \alpha \tau c}{n}
    e^{c \cdot (\alpha - 1)_+},
  \]
  where we use the notation $(x)_+ = \max(0, x)$,
  the $\omega$-sparse
  estimation time of the \hogwild chain can be bounded with
  \[
    t_{\mathrm{SE-hog}(\omega)}(\epsilon)
    \le
    \left\lceil
    \bar t_{\mathrm{SE-seq}(\omega)}(\epsilon)
    +
    \frac{
      2 \omega \alpha \tau c^2
    }{
      \epsilon
    }
    e^{c \cdot (\alpha - 1)_+}
    \right\rceil.
  \]
\end{restatable}

\section{Proofs}

Here, we provide proofs for the results in the paper.  In the first
subsection, we will state lemmas and known results that we will use in the
subsequent proofs.  Next, we will prove the Claims and Theorems stated in the
body of the paper.  Finally, we will prove the lemmas previously stated.

\subsection{Statements of Lemmas}

First, we state a proposition from ~\citetsec{levin2009markov}.  This proposition
relates the concept of a coupling with the total variation distance between
the distributions of two random variables.

\begin{proposition}[Proposition 4.7 from ~\citetsec{levin2009markov}]
  \label{propCoupling}
  Let $X$ and $Y$ be two random variables that take on values in the same set,
  and let their distributions be $\mu$ and $\nu$, respectively.
  Then for any coupling, $(\bar X, \bar Y)$ it will hold that
  \[
    \tvdist{\mu - \nu} \le \Prob{\bar X \ne \bar Y}.
  \]
  Furthermore, there exists a coupling for which equality is achieved; this
  is called an \emph{optimal} coupling.
\end{proposition}

We can prove a related result for sparse variation distance.

\begin{restatable}{rlemma}{lemmaSparseCoupling}
  \label{lemmaSparseCoupling}
  Let $X$ and $Y$ be two random variables that each assign values to a set of
  variables $\{1,\ldots,n\}$,
  and let their distributions be $\mu$ and $\nu$, respectively.
  Then for any coupling, $(\bar X, \bar Y)$ it will hold that
  \[
    \svdist{\mu - \nu}{\omega}
    \le
    \max_{I \subseteq \{1,\ldots,n\}, \, \Abs{I} \le \omega}
    \Prob{\exists i \in I, \: \bar X_i \ne \bar Y_i}.
  \]
\end{restatable}

We state a lemma that bounds the expected total variation distance
between the marginal distributions of two states using the total
influence $\alpha$.  \crcchange{Note that a similar statement to that proved
in this lemma may be used as
an alternate definition for the total influence $\alpha$; the definition given
in the body of the paper is used because it is more intuitive and does not
require introducing the concept of a coupling.} 
This lemma will be useful later when proving the
subsequent lemmas stated in this subsection.

\begin{restatable}{rlemma}{lemmaLocalInfExvDist}
  \label{lemmaLocalInfExvDist}
  If $\pi$ is a distribution with total influence
  $\alpha$, and $X$ and $Y$ are two random variables that take on values in the
  state space of $\pi$, then for any variable $i$
  \[
    \Exv{
      \tvdist{
        \pi_i(\cdot | X)
        -
        \pi_i(\cdot | Y)
      }
    }
    \le
    \alpha
    \max_j \Prob{X_j \ne Y_j},
  \]
  \crcchange{where, for simplicity of notation, we let
  $\pi_i(\cdot | X)$ denote the conditional distribution of variable $i$
  in $\pi$ given the values of all the other variables in state $X$.}
\end{restatable}

Next, we state three lemmas, each of which give bounds on the quantity
\[
  \Prob{X_{i,t} \ne Y_{i,t}}
\]
for some coupling of two (potentially asynchronous) Gibbs sampling chains.
First, we state the result for comparing two synchronous chains.

\begin{restatable}{rlemma}{lemmaSequentialLocalPB}
  \label{lemmaSequentialLocalPB}
  Consider sequential Gibbs sampling on a distribution $\pi$ with total
  influence $\alpha$.  Then, for any initial states $(X_0, Y_0)$ there exists a 
  coupling of the chains $(X_t, Y_t)$ such that for any variable $i$ and any
  time $t$,
  \[
    \Prob{X_{i,t} \ne Y_{i,t}}
    \le
    \exp\left( - \frac{1 - \alpha}{n} t \right).
  \]
\end{restatable}

Second, we state the result comparing two \hogwild chains.

\begin{restatable}{rlemma}{lemmaHogwildLocalPB}
  \label{lemmaHogwildLocalPB}
  Consider any model of \hogwild-Gibbs sampling on a distribution $\pi$ with
  total influence $\alpha$.  Then, for any initial states $(X_0, Y_0)$ there
  exists a coupling  $(X_t, Y_t)$ of the \hogwild-Gibbs sampling chains
  starting at $X_0$ and $Y_0$ respectively such that for any variable $i$
  and any time $t$,
  \[
    \Prob{X_{i,t} \ne Y_{i,t}}
    \le
    \exp\left(- \frac{1 - \alpha}{n + \alpha \tau^*} t \right).
  \]
\end{restatable}

Third, we state the result comparing a sequential and an asynchronous chain.

\begin{restatable}{rlemma}{lemmaHogwildLocalBias}
  \label{lemmaHogwildLocalBias}
  Consider any model of \hogwild-Gibbs sampling on a distribution $\pi$ with
  total influence $\alpha$.  Then if for any initial states $(X_0, Y_0)$ we can
  construct a
  coupling $(X_t, Y_t)$ such that the process $X_t$ is distributed according
  to the dynamics of \hogwild-Gibbs, the process $Y_t$ is distributed
  according to the dynamics of sequential Gibbs, and for any time $t$,
  \[
    \max_i
    \Prob{X_{i,t+1} \ne Y_{i,t+1}}
    \le
    \left( 1 - \frac{1 - \alpha}{n} \right)
    \max_i
    \Prob{X_{i,t} \ne Y_{i,t}}
    +
    \frac{\alpha \tau}{n^2}.
  \]
  As a secondary result, if the chain satisfies Dobrushin's condition
  ($\alpha < 1$), then for any variable $i$ and any time $t$,
  \[
    \Prob{X_{i,t} \ne Y_{i,t}}
    \le
    \exp\left( - \frac{1 - \alpha}{n} t \right)
    +
    \frac{\alpha \tau}{(1 - \alpha) n}.
  \]
\end{restatable}

% To prove that our catastrophic mixing time example mixes in $O(n \log n)$ time
% in the sequential case, we will need to use the following result, which is
% inspired by work on mixing using strong stationary times
% in \citetsec{diaconis1981generating} and \citetsec{diaconis1988group}.

% \begin{lemma}[Corollary 8.10 from \citet{levin2009markov}]
%   \label{lemmaDiaconisShuffle}
%   Consider the Markov chain over permutations of the set $\{1, \ldots, n \}$
%   that, at each step, swaps two uniformly-chosen random elements of the
%   permutation.  This Markov chain has mixing time
%   \[
%     t_{\mathrm{mix}}
%     \le
%     (2 + o(1)) n \log n
%     \le
%     O(n \log n).
%   \]
% \end{lemma}

% This lemma has, as a direct consequence, the following, more
% directly-applicable, lemmas.

% \begin{lemma}
%   \label{lemma1Shuffle}
%   Consider the Markov chain over the state space
%   \[
%     \Omega
%     =
%     \left\{
%       X \in \{ -1, 1 \}^N \middle| \sum_i X_i = c
%     \right\}
%   \]
%   for some constant $c$, with dynamics
%   that, at each step, swap two uniformly-chosen random elements of the
%   vector.  This Markov chain has mixing time
%   \[
%     t_{\mathrm{mix}}
%     \le
%     (2 + o(1)) n \log n
%     \le
%     O(n \log n).
%   \]
% \end{lemma}

\begin{restatable}[Monotonic Sequence Domination Lemma]{rlemma}{lemmaMSDL}
  \label{lemmaMSDL}
  Let $x_0, x_1, \ldots$ be a sequence such that, for all $t$,
  \[
    x_{t+1} \le f_t(x_t, x_{t-1}, \ldots, x_0),
  \]
  where $f_t$ is a function that is monotonically increasing in all of its
  arguments.  Then, for any sequence $y_0, y_1, \ldots$, if $x_0 = y_0$ and
  for all $t$,
  \[
    y_{t+1} \ge f_t(y_t, y_{t-1}, \ldots, y_0),
  \]
  then for all $t$,
  \[
    x_t \le y_t.
  \]
\end{restatable}

\begin{restatable}{rlemma}{lemmaModelX}
  \label{lemmaModelX}
  Consider the model on $N$ variables $X_i$, for $N$ odd, where each $X_i$
  takes on values in $\{ -1, 1 \}$ and has probability
  \[
    \pi(X)
    =
    \frac{1}{Z_X}
    \left\{ 
      \begin{array}{l l}
        1 & \text{if } \Abs{\mathbf{1}^T X} = 1 \\
        0 & \text{if } \Abs{\mathbf{1}^T X} > 1
      \end{array}
    \right.
  \]
  Then Gibbs sampling on this model \crcchange{(assuming that we allow the
  chain to start only at a state $X$ where $\pi(X) > 0$)} has mixing time
  \[
    t_{\mathrm{mix}} = O(n \log n).
  \]
\end{restatable}

\subsection{Proofs of Bias Results}

First, we restate and
prove Claim \ref{claimHogwildGeneralBiasBigO}.  This proof will use
the result of Theorem \ref{thmHogwildGeneralBias}, which we will prove
subsequently.  We note here that the use of a convex upper bound for the
sparse estimation time of the sequential chain (as opposed to using the
sequential chain's sparse estimation time directly) is an unfortunate
consequence of the proof---we hope that a more careful analysis could remove
it or replace it with a more natural condition.

\claimHogwildGeneralBiasBigO*

\begin{proof}

  First, note that, since $\alpha = O(1)$, we know by the definition
  of big-$O$ notation that for some $\alpha^*$, for all models
  in the class, the total influence of that model will be
  $\alpha \le \alpha^*$.
  Similarly, since we assumed that, for any $\epsilon$ and across
  all models $\pi$,
  \[
    \bar t_{\mathrm{SM-seq}(\omega)}(\pi, \epsilon) = O(n),
  \]
  then for each $\epsilon$, there must exist a $c(\epsilon)$
  such that for any distribution $\pi$ with $n$ variables in the class,
  \[
    t_{\mathrm{SM-seq}(\omega)}(\pi, \epsilon) \le n \cdot c(\epsilon).
  \]

  For some error $\epsilon$ and model $\pi$, we would like to apply Theorem
  \ref{thmHogwildGeneralBias} to bound its mixing time.  In order to apply
  the theorem, we must satisfy the conditions on $\epsilon$: it suffices for
  \[
    n
    \ge
    \frac{2 \omega \alpha^* \tau c(\epsilon / 2)}{\epsilon}
    \exp\left( (\alpha^* - 1)_+ c(\epsilon / 2) \right).
  \]
  Under this condition, applying the theorem allows us to bound the
  $\omega$-sparse estimation time of the \hogwild chain with
  \begin{dmath*}
    t_{\mathrm{SE-hog}(\omega)}(\epsilon)
    \le
    \left\lceil
    \bar t_{\mathrm{SE-seq}(\omega)}(\epsilon)
    +
    \frac{
      2 \omega \alpha^* \tau c(\epsilon / 2)^2
    }{
      \epsilon
    }
    \exp\left((\alpha^* - 1)_+ c(\epsilon / 2) \right)
    \right\rceil
    \le
    \bar t_{\mathrm{SE-seq}(\omega)}(\epsilon)
    +
    \frac{
      2 \omega \alpha^* \tau c(\epsilon / 2)^2
    }{
      \epsilon
    }
    \exp\left((\alpha^* - 1)_+ c(\epsilon / 2) \right)
    +
    1
  \end{dmath*}
  Therefore, if we define
  \[
    N(\epsilon)
    =
    \frac{2 \omega \alpha^* \tau c(\epsilon / 2)}{\epsilon}
    \exp\left( (\alpha^* - 1)_+ c(\epsilon / 2) \right),
  \]
  and
  \[
    T(\epsilon)
    =
    \frac{
      2 \omega \alpha^* \tau c(\epsilon / 2)^2
    }{
      \epsilon
    }
    \exp\left((\alpha^* - 1)_+ c(\epsilon / 2) \right)
    +
    1,
  \]
  then it follows that, for any $\epsilon$ and for all models with
  $n \ge N(\epsilon)$,
  \[
    t_{\mathrm{SM-hog}(\omega)}(\epsilon)
    \le
    t_{\mathrm{SM-seq}(\omega)}(\epsilon)
    +
    T(\epsilon).
  \]
  This is equivalent to saying that, for any $\epsilon$ and across all models,
  \[
    t_{\mathrm{SM-hog}(\omega)}(\epsilon)
    \le
    t_{\mathrm{SM-seq}(\omega)}(\epsilon)
    +
    O(1).
  \]
  This proves the claim.
\end{proof}

Next, we restate and prove the bias lemma, Lemma \ref{lemmaHogwildBiasRFive}.

\lemmaHogwildBiasRFive*

\begin{proof}[Proof of Lemma \ref{lemmaHogwildBiasRFive}]
  We start by using the primary result from Lemma \ref{lemmaHogwildLocalBias}.
  This result states that we can construct a coupling $(X_t, Y_t)$ of the
  \hogwild and sequential chains starting at any initial distributions $X_0$
  and $Y_0$ such that at any time $t$,
  \[
    \max_i
    \Prob{X_{i,t+1} \ne Y_{i,t+1}}
    \le
    \left( 1 - \frac{1 - \alpha}{n} \right)
    \max_i
    \Prob{X_{i,t} \ne Y_{i,t}}
    +
    \frac{\alpha \tau}{n^2}.
  \]
  Now, for any initial distribution $\mu_0$, \crcchange{assume that we start
  with $X_0 = Y_0$, where both are distributed according to $\mu_0$}.  Then,
  trivially,
  \[
    \Prob{X_{i,0} \ne Y_{i,0}} = 0.
  \]
  It follows from recursive application of the sub-result of Lemma
  \ref{lemmaHogwildLocalBias} that, for this coupling,
  \begin{dmath*}
    \max_i
    \Prob{X_{i,t} \ne Y_{i,t}}
    \le
    \sum_{k=0}^{t-1}
    \left( 1 + \frac{\alpha - 1}{n} \right)^k
    \frac{\alpha \tau}{n^2}
    \le
    t \left( 1 + \frac{(\alpha - 1)_+}{n} \right)^t
    \frac{\alpha \tau}{n^2}
    \le
    \exp\left( \frac{(\alpha - 1)_+}{n} t \right)
    \frac{\alpha \tau t}{n^2},
  \end{dmath*}
  where $(x)_+$ denotes $\max(0, x)$.
  It follows by the union bound that, for any set
  of variables $I$ with $\Abs{I} \le \omega$, the probability that the coupling
  is unequal in at least one of those variables is
  \begin{dmath*}
    \Prob{\exists i \in I, \: X_{i,t} \ne Y_{i,t}}
    \le
    \omega
    \max_i
    \Prob{X_{i,t} \ne Y_{i,t}}
    \le
    \exp\left( \frac{(\alpha - 1)_+}{n} t \right)
    \frac{\omega \alpha \tau t}{n^2}.
  \end{dmath*}
  Since this inequality holds for any set of variable $I$ with
  $\Abs{I} \le \omega$, it follows that
  \begin{dmath*}
    \max_{I \subseteq \{1,\ldots,n\}, \, \Abs{I} \le \omega}
    \Prob{\exists i \in I, \: X_{i,t} \ne Y_{i,t}}
    \le
    \omega
    \max_i
    \Prob{X_{i,t} \ne Y_{i,t}}
    \le
    \exp\left( \frac{(\alpha - 1)_+}{n} t \right)
    \frac{\omega \alpha \tau t}{n^2}.
  \end{dmath*}
  We can proceed to apply Lemma \ref{lemmaSparseCoupling}, which lets us
  conclude that
  \begin{dmath*}
    \svdist{
      P^{(t)}_{\mathrm{hog}} \mu_t
      -
      P^{(t)}_{\mathrm{seq}} \nu_t
    }{\omega}
    \le
    \omega
    \max_i
    \Prob{X_{i,t} \ne Y_{i,t}}
    \le
    \exp\left( \frac{(\alpha - 1)_+}{n} t \right)
    \frac{\omega \alpha \tau t}{n^2}
  \end{dmath*}
  This is the desired result.
\end{proof}

Next, we restate and prove
the full bias result, Theorem \ref{thmHogwildGeneralBias}.

\thmHogwildGeneralBias*

\begin{proof}[Proof of Theorem \ref{thmHogwildGeneralBias}]
  We start with the result of Lemma \ref{lemmaHogwildBiasRFive}, which
  lets us conclude that
  \[
    \svdist{\mu_t - \nu_t}{\omega}
    \le
    \exp\left( \frac{(\alpha - 1)_+}{n} t \right)
    \frac{\omega \alpha \tau t}{n^2}
  \]
  \crcchange{where $\mu_t = P^{(t)}_{\mathrm{hog}} \mu_0$ and
  $\nu_t = P^{(t)}_{\mathrm{seq}} \mu_0$ are the distributions of the
  \hogwild and
  sequential Gibbs sampling chains}, respectively, starting in state $\mu_0$.
  Next, since $\nu_t$ has the dynamics of the sequential Gibbs sampling chain,
  and $\bar t_{\mathrm{SM-seq}(\omega)}(\epsilon)$ is an upper bound for
  the sparse estimation time, it follows that for any $\epsilon$, if
  \[
    t \ge \bar t_{\mathrm{SM-seq}(\omega)}(\epsilon),
  \]
  then
  \[
    \svdist{\nu_t - \pi}{\omega}
    \le
    \epsilon.
  \]
  Since $\bar t_{\mathrm{SM-seq}(\omega)}(\epsilon)$ is a decreasing
  function of $\epsilon$, it must have an inverse function.  Furthermore, since
  it is convex, its inverse function must also be convex.
  Therefore, we can also write the above expression in terms of the inverse
  function; for any $t$,
  \[
    \svdist{\nu_t - \pi}{\omega}
    \le
    \bar t_{\mathrm{SM-seq}(\omega)}^{-1}(t).
  \]
  Therefore, by the triangle inequality, for any $t$,
  \begin{dmath*}
    \svdist{\mu_t - \pi}{\omega}
    \le
    \svdist{\mu_t - \nu_t}{\omega}
    +
    \svdist{\nu_t - \pi}{\omega}
    \le
    \frac{\omega \alpha \tau t}{n^2}
    \exp\left( \frac{(\alpha - 1)_+}{n} t \right)
    +
    \bar t_{\mathrm{SM-seq}(\omega)}^{-1}(t).
  \end{dmath*}
  Now, for any particular $\epsilon$, let
  \[
    t_0 = \bar t_{\mathrm{SM-seq}(\omega)}(\epsilon),
  \]
  and let
  \[
    t_1 = \bar t_{\mathrm{SM-seq}(\omega)}\left( \frac{\epsilon}{2} \right).
  \]
  Further define
  \[
    R
    =
    \frac{\omega \alpha \tau t_1}{n^2}
    \exp\left( \frac{(\alpha - 1)_+}{n} t_1 \right).
  \]
  Therefore, for any $t_0 \le t \le t_1$,
  \begin{dmath*}
    \svdist{\mu_t - \pi}{\omega}
    \le
    R
    +
    \bar t_{\mathrm{SM-seq}(\omega)}^{-1}(t).
  \end{dmath*}
  By convexity of $\bar t_{\mathrm{SM-seq}(\omega)}^{-1}$, we can bound this
  expression over the interval $t_0 \le t \le t_1$ with
  \begin{dmath*}
    \svdist{\mu_t - \pi}{\omega}
    \le
    R
    +
    \frac{t_1 - t}{t_1 - t_0} \cdot \epsilon
    +
    \frac{t - t_0}{t_1 - t_0} \cdot \frac{\epsilon}{2},
  \end{dmath*}
  and so, if we want this to be less than $\epsilon$, it suffices to choose
  $t$ such that
  \[
    \epsilon
    =
    R
    +
    \frac{t_1 - t}{t_1 - t_0} \cdot \epsilon
    +
    \frac{t - t_0}{t_1 - t_0} \cdot \frac{\epsilon}{2}
  \]
  which will occur when
  \[
    t
    =
    t_0 + \frac{2 R (t_1 - t_0)}{\epsilon}.
  \]
  Now, applying the definition
  \[
    c
    =
    \frac{1}{n}
    \bar t_{\mathrm{SM-seq}(\omega)}\left(\frac{\epsilon}{2}\right)
    =
    \frac{t_1}{n}
  \]
  lets us equivalently write $R$ as
  \[
    R
    =
    \frac{\omega \alpha \tau c}{n}
    \exp\left( \frac{(\alpha - 1)_+}{n} t_1 \right).
  \]
  Recall that as a condition for the theorem, we assumed that
  \[
    \epsilon
    \ge
    \frac{2 \omega \alpha \tau c}{n}
    \exp\left( \frac{(\alpha - 1)_+}{n} t_1 \right).
  \]
  It follows from this and our expression for $R$ that
  \[
    R \le \frac{\epsilon}{2}.
  \]
  Therefore this assignment of $t$ will satisfy the previous constraint that
  $t_0 \le t \le t_1$, and so for this assignment of $t$, and for any initial
  distribution $\mu_0$, it holds that
  \begin{dmath*}
    \svdist{\mu_t - \pi}{\omega}
    \le
    \epsilon.
  \end{dmath*}
  Therefore, by the definition of sparse estimation time, the sparse estimation
  time of the \hogwild chain will be
  \[
    t_{\mathsf{SE-hog}(\omega)}(\epsilon)
    \le
    t,
  \]
  for this assignment of $t$.  Now, recall that above we assigned
  \[
    t
    =
    t_0 + \frac{2 R (t_1 - t_0)}{\epsilon}.
  \]
  Under this condition, we can bound this whole error term as
  \[
    \frac{2 R (t_1 - t_0)}{\epsilon}
    \le
    \frac{2 \omega \alpha \tau t_1^2}{n^2 \epsilon}
    \exp\left( \frac{(\alpha - 1)_+}{n} t_1 \right).
  \]
  Combining this with the definitions of $t_0$ and $c$ lets us state that
  \[
    t
    \le
    \bar t_{\mathrm{SM-seq}(\omega)}(\epsilon)
    +
    \frac{
      2 \omega \alpha \tau c^2
    }{
      \epsilon
    }
    \exp\left((\alpha - 1)_+ c \right).
  \]
  Taking the ceiling implies that, when
  \[
    t
    =
    \left\lceil
      \bar t_{\mathrm{SM-seq}(\omega)}(\epsilon)
      +
      \frac{
        2 \omega \alpha \tau c^2
      }{
        \epsilon
      }
      \exp\left((\alpha - 1)_+ c \right)
    \right\rceil,
  \]
  for any initial distribution $\mu_0$,
  \[
    \svdist{\mu_t - \pi}{\omega} \le \epsilon.
  \]
  \crcchange{
  Since we above defined $\mu_t$ to be the distribution of \hogwild Gibbs
  after $t$ timesteps, $\mu_t = P^{(t)} \mu_0$, where $P^{(t)}$ is the
  transition matrix of \hogwild Gibbs after $t$ timesteps.  We can 
  thus equivalently write this as}
  \[
    \svdist{P^{(t)} \mu_0 - \pi}{\omega} \le \epsilon.
  \]
  Therefore, by the definition of sparse estimation time,
  \[
    t_{\mathsf{SE-hog}}(\epsilon)
    \le
    \left\lceil
      \bar t_{\mathrm{SM-seq}(\omega)}(\epsilon)
      +
      \frac{
        2 \omega \alpha \tau c^2
      }{
        \epsilon
      }
      \exp\left((\alpha - 1)_+ c \right)
    \right\rceil.
  \]
  This proves the theorem.
\end{proof}

Next, we restate and prove the
theorem that bounds the sparse estimation time of sequential
Gibbs for distributions that satisfy Dobrushin's condition.

\thmSequentialLocalBias*

\begin{proof}[Proof of Theorem \ref{thmSequentialLocalBias}]
  We start by using the result of Lemma \ref{lemmaSequentialLocalPB}.  This
  result states that, for any initial distributions $(X_0, Y_0)$, there exists
  a coupling $(X_t, Y_t)$ of the sequential Gibbs sampling chains starting at
  distributions $X_0$ and $Y_0$, respectively, such that for any variable $i$
  and any time $t$,
  \[
    \Prob{X_{i,t} \ne Y_{i,t}}
    \le
    \exp\left( - \frac{1 - \alpha}{n} t \right).
  \]
  It follows by the union bound that, for any set of variables $I$ with
  $\Abs{I} \le \omega$, the
  probability that the coupling is unequal in at least one of those variables
  is
  \[
    \Prob{\exists i \in I, \: X_{i,t} \ne Y_{i,t}}
    \le
    \omega
    \exp\left( - \frac{1 - \alpha}{n} t \right).
  \]
  Since this inequality holds for any set of variable $I$ with
  $\Abs{I} \le \omega$, it follows that
  \begin{dmath*}
    \max_{I \subseteq \{1,\ldots,n\}, \, \Abs{I} \le \omega}
    \Prob{\exists i \in I, \: X_{i,t} \ne Y_{i,t}}
    \le
    \omega
    \exp\left( - \frac{1 - \alpha}{n} t \right).
  \end{dmath*}
  We can proceed to apply Lemma \ref{lemmaSparseCoupling}, which lets us
  conclude that, \crcchange{if we let $\mu_t$ and $\nu_t$ denote the
  distributions of $X_t$ and $Y_t$, respectively, then}
  \begin{dmath*}
    \svdist{\mu_t - \nu_t}{\omega}
    \le
    \omega
    \exp\left( - \frac{1 - \alpha}{n} t \right).
  \end{dmath*}
  Since this was true for any initial distributions for $X_0$ and $Y_0$,
  it will hold in particular for $Y_0$ distributed according to $\pi$, the
  stationary distribution of the chain.  In this case, $\nu_t = \pi$, and 
  so for any initial distribution $\mu_0$ for $X_0$,
  \begin{dmath*}
    \svdist{\mu_t - \pi}{\omega}
    \le
    \omega
    \exp\left( - \frac{1 - \alpha}{n} t \right).
  \end{dmath*}
  Now, in order for this to be bounded by $\epsilon$, it suffices to choose
  $t$ such that
  \[
    \omega \exp\left( - \frac{1 - \alpha}{n} t \right)
    \le
    \epsilon.
  \]
  This will occur whenever
  \[
    t
    \ge
    \frac{n}{1 - \alpha}
    \log\left(\frac{\omega}{\epsilon}\right)
  \]
  (here we used the fact that $\alpha < 1$ to do the division).
  Taking the ceiling, we can conclude that when
  \[
    t
    =
    \left\lceil
    \frac{n}{1 - \alpha}
    \log\left(\frac{\omega}{\epsilon}\right)
    \right\rceil.
  \]
  for any initial distribution $\mu_0$,
  \[
    \svdist{\mu_t - \pi}{\omega} \le \epsilon.
  \]
  \crcchange{
  Since we defined $\mu_t$ to be the distribution of $X_t$, it must hold that
  $\mu_t = P^{(t)} \mu_0$, where $\mu_0$ is the initial distribution of $X_0$,
  and $P^{(t)}$ is the transition matrix associated with
  running $t$ steps of sequential Gibbs sampling.  Thus, we can rewrite this
  as}
  \[
    \svdist{P^{(t)} \mu_0 - \pi}{\omega} \le \epsilon.
  \]
  Since this result held for any initial assignment of $X_0$ and therefore for
  any $\mu_0$, by the definition of sparse estimation time it follows that
  \[
    t_{\mathsf{SE-seq}}(\epsilon)
    \le
    \left\lceil
    \frac{n}{1 - \alpha}
    \log\left(\frac{\omega}{\epsilon}\right)
    \right\rceil.
  \]
  This proves the theorem.
\end{proof}

Next, we restate and prove
the theorem that bounds the sparse estimation time of \hogwild
Gibbs for distributions that satisfy Dobrushin's condition.

\thmHogwildLocalBias*

\begin{proof}[Proof of Theorem \ref{thmHogwildLocalBias}]
  We start by using the secondary result from
  Lemma \ref{lemmaHogwildLocalBias}---we can safely use this result because
  we assumed the chain satisfied Dobrushin's condition ($\alpha < 1$).
  This result states that we can construct a coupling $(X_t, Y_t)$ of the
  \hogwild and sequential chains starting at any initial distributions $X_0$
  and $Y_0$ such that at any time $t$,
  \[
    \Prob{X_{i,t} \ne Y_{i,t}}
    \le
    \exp\left( - \frac{1 - \alpha}{n} t \right)
    +
    \frac{\alpha \tau}{(1 - \alpha) n}.
  \]
  It follows by the union bound that, for any set of variables $I$ with
  $\Abs{I} \le \omega$, the
  probability that the coupling is unequal in at least one of those variables
  is
  \[
    \Prob{\exists i \in I, \: X_{i,t} \ne Y_{i,t}}
    \le
    \omega \exp\left( - \frac{1 - \alpha}{n} t \right)
    +
    \frac{\omega \alpha \tau}{(1 - \alpha) n}
  \]
  Since this inequality holds for any set of variable $I$ with
  $\Abs{I} \le \omega$, it follows that
  \begin{dmath*}
    \max_{I \subseteq \{1,\ldots,n\}, \, \Abs{I} \le \omega}
    \Prob{\exists i \in I, \: X_{i,t} \ne Y_{i,t}}
    \le
    \omega \exp\left( - \frac{1 - \alpha}{n} t \right)
    +
    \frac{\omega \alpha \tau}{(1 - \alpha) n}.
  \end{dmath*}
  We can proceed to apply Lemma \ref{lemmaSparseCoupling}, which lets us
  conclude that, \crcchange{if we let $\mu_t$ and $\nu_t$ denote the
  distributions of $X_t$ and $Y_t$ respectively,}
  \begin{dmath*}
    \svdist{\mu_t - \nu_t}{\omega}
    \le
    \omega \exp\left( - \frac{1 - \alpha}{n} t \right)
    +
    \frac{\omega \alpha \tau}{(1 - \alpha) n}.
  \end{dmath*}
  To bound the sparse estimation time, notice that for any fixed $\epsilon$
  (independent of $n$), in order to achieve
  \[
    \svdist{\mu_t - \pi}{\omega} \le \epsilon,
  \]
  it suffices to choose any $t$ such that
  \[
    \omega \exp\left( - \frac{1 - \alpha}{n} t \right)
    \le
    \epsilon
    -
    \frac{\omega \alpha \tau}{(1 - \alpha) n}.
  \]
  This will occur when
  \[
    \frac{1 - \alpha}{n} t
    \ge
    \log\left( \frac{\omega}{\epsilon} \right)
    -
    \log\left(
      1
      -
      \frac{\omega \alpha \tau}{(1 - \alpha) n \epsilon}
    \right).
  \]
  Next, recall that we assumed that
  \[
    \epsilon \ge \frac{2 \omega \alpha \tau}{(1 - \alpha) n};
  \]
  therefore $\epsilon$ is large enough that
  \[
    \frac{\omega \alpha \tau}{(1 - \alpha) n \epsilon} \le \frac{1}{2}.
  \]
  It is easy to prove that, for all $x \le \frac{1}{2}$,
  \[
    \log(1 - x) \ge 2x.
  \]
  Therefore, under this condition in $\epsilon$, it suffices to choose $t$ such
  that
  \[
    \frac{1 - \alpha}{n} t
    \ge
    \log\left( \frac{\omega}{\epsilon} \right)
    +
    \frac{2 \omega \alpha \tau}{(1 - \alpha) n \epsilon};
  \]
  this will occur whenever
  \[
    t
    \ge
    \frac{n}{1 - \alpha}
    \log\left( \frac{\omega}{\epsilon} \right)
    +
    \frac{2 \omega \alpha \tau}{(1 - \alpha)^2 \epsilon}.
  \]
  Taking the ceiling implies that, when
  \[
    t
    =
    \left\lceil
      \frac{n}{1 - \alpha}
      \log\left( \frac{\omega}{\epsilon} \right)
      +
      \frac{2 \omega \alpha \tau}{(1 - \alpha)^2 \epsilon}
    \right\rceil,
  \]
  for any initial distribution $\mu_0$,
  \[
    \svdist{\mu_t - \pi}{\omega} \le \epsilon.
  \]
  \crcchange{
  Since we defined $\mu_t$ above to be the distribution of $X_t$, it follows
  that $\mu_t = P^{(t)} \mu_0$, where $\mu_0$ is the initial distribution of
  $X_0$ and $P^{(t)}$ is the transition matrix associated with running $t$
  steps of \hogwild Gibbs.  Therefore, we can rewrite this as
  }
  \[
    \svdist{P^{(t)} \mu_0 - \pi}{\omega} \le \epsilon.
  \]
  Since this is true for any initial distribution of $X_0$ and therefore for
  any $\mu_0$, it follows from the definition of sparse estimation time that
  \[
    t_{\mathsf{SE-hog}}(\epsilon)
    \le
    \left\lceil
      \frac{n}{1 - \alpha}
      \log\left( \frac{\omega}{\epsilon} \right)
      +
      \frac{2 \omega \alpha \tau}{(1 - \alpha)^2 \epsilon}
    \right\rceil.
  \]
  This proves the theorem.
\end{proof}

\subsection{Proofs of Mixing Time Results}

First, we restate and prove Statement \ref{stmtBadMixExample}.

\stmtBadMixExample*

\begin{proof}[Proof of Statement \ref{stmtBadMixExample}]
  We start out by proving that the model mixes rapidly in the sequential
  case.

  First, we assume that we select $M_1$ large enough that, even for potentially
  exponential run times, the dynamics of the chain are indistinguishable 
  from the chain with $M_1 = \infty$.  In particular, this alternate chain
  will have the following properties:
  \begin{itemize}
    \item The dynamics of the $X$ part of the chain do not depend in any way
      on the value of $Y$.
    \item If at any point, $\Abs{\mathbf{1}^T X} > 1$, whenever we sample an
      $X$ variable, we will re-sample it if possible to decrease the value
      of $\Abs{\mathbf{1}^T X}$ with probability $1$.
    \item As long as $\Abs{\mathbf{1}^T X} = 1$ at some point in time, this
      will remain true, and the dynamics of the $X$ part of the chain will
      be those of the chain described in Lemma \ref{lemmaModelX}.
  \end{itemize}
  We assume that we choose $M_1$ large enough that these properties hold over
  all time windows discussed in this proof with high probability.

  Now, by the coupon collector's problem, after $O(N \log N)$ timesteps,
  we have sampled all the variables with high probability.  If we have sampled
  all the variables with high probability, then we will certainly have
  $\Abs{\mathbf{1}^T X} = 1$ with high probability.

  Once we have $\Abs{\mathbf{1}^T X} = 1$, Lemma \ref{lemmaModelX} ensures
  that, after $O(N \log N)$ additional timesteps, the $X$ part of the chain
  will be close to its stationary distribution.

  Meanwhile, while $\Abs{\mathbf{1}^T X} = 1$, the dynamics of the $Y$ part
  of the chain are exactly Gibbs sampling over the model with energy
  \[
    \phi_Y(Y)
    =
    \frac{\beta}{N} \left( \mathbf{1}^T Y \right)^2.
  \]
  For any $\beta < 1$, this is known to mix in $O(N \log N)$ time, since it
  satisfies Dobrushin's condition.  Therefore, after $O(N \log N)$ steps
  after we have $\Abs{\mathbf{1}^T X} = 1$, the $Y$ part of the chain will
  also be close to its stationary distribution.

  Summing up the times for the above events gives us a total mixing time for
  this chain of 
  \[
    t_{\mathrm{mix-seq}} = O(N \log N).
  \]

  Next we prove that the model takes a potentially exponential time to mix
  in the asynchronous case.  Assume here that our model of execution has
  two threads, which always either sample two $X$ variables independently
  and asynchronously, or sample a single $Y$ variable synchronously (i.e.
  there is never any delay when reading the value of a $Y$ variable).
  For this execution pattern, we have uniformly that
  $\tau_{i,t} \le 1$.  In particular, this has $\tau = O(1)$.

  Now, consider the case where the two threads each choose to sample a variable
  in $X$ that can be switched.  Since at least $\frac{1}{4}$ of the variables
  are variables in $X$ that can be switched, this will occur with probability
  at least $\frac{1}{16}$.
  Given this, they will each independently switch their variable with
  probability $\frac{1}{2}$.  This means that both variables are switched
  with probability $\frac{1}{4}$ --- but this would place the system in a
  state where
  \[
    \Abs{\mathbf{1}^T X} > 1.
  \]
  At any time when $\Abs{\mathbf{1}^T X} = 1$, this will occur with probability
  $\frac{1}{64}$, which implies that whenever we sample $Y$, the probability
  that $\Abs{\mathbf{1}^T X} > 1$ is at least $\frac{1}{64}$.

  Now, assume without loss of generality that we initialize $Y$ such that
  $\mathbf{1}^T Y = N$.  Let $\rho_t$ denote the value of $\mathbf{1}^T Y$ at
  time $t$.  Assuming that we sample a variable $Y_i$ with value $1$,
  while $\Abs{\mathbf{1}^T X} = 1$,
  the probability that it will be switched will be
  \begin{dmath*}
    \Prob{\text{value switched}}
    =
    \frac{
      \exp\left( \beta n^{-1} (\rho_t - 1)^2 \right)
    }{
      \exp\left( \beta n^{-1} (\rho_t - 1)^2 \right)
      +
      \exp\left( \beta n^{-1} (\rho_t)^2 \right)
    }
    =
    \left(
      1
      +
      \exp\left( \beta n^{-1} \left( (\rho_t)^2 - (\rho_t - 1)^2 \right) \right)
    \right)^{-1}
    =
    \left(
      1
      +
      \exp\left( \beta n^{-1} ( 2 \rho_t - 1 ) \right)
    \right)^{-1}.
  \end{dmath*}
  Note that since $\rho_t \le N$ at all times, if $\beta < 1$,
  \[
    \beta N^{-1} ( 2 \rho_t - 1 )
    \le
    2.
  \]
  We also can verify that, for any $0 \le x \le 2$, as a basic property 
  of the exponential function,
  \[
    \left( 1 + \exp(x) \right)^{-1}
    \le
    \frac{1}{2} - \frac{x}{6}.
  \]
  Therefore, as long as $\rho_t > 0$,
  \begin{dmath*}
    \Prob{\text{value switched}}
    \le
    \frac{1}{2}
    -
    \frac{\beta \rho_t}{3 n}.
  \end{dmath*}
  Therefore, as long as $\rho_t > 0$, and $\Abs{\mathbf{1}^T X} = 1$,
  \begin{dmath*}
    \Exvc{\rho_{t+1}}{\F_t}
    \ge
    \rho_t
    +
    2 \left(
      \frac{N - \rho_t}{2 N}
      -
      \frac{1}{2}
      +
      \frac{\beta \rho_t}{3 N}
    \right)
    =
    \rho_t
    +
    2 \left(
      \frac{-\rho_t}{2 N}
      +
      \frac{\beta \rho_t}{3 N}
    \right)
    =
    \rho_t \left(
      1
      -
      \frac{3 - 2 \beta}{3 N}
    \right).
  \end{dmath*}
  On the other hand, if $\Abs{\mathbf{1}^T X} > 1$, then we can pick $M_2$
  large enough such that with high probability, as long as $\rho_t > 0$,
  all variables $Y_i$ are always sampled to be $1$.  In this case,
  \begin{dmath*}
    \Exvc{\rho_{t+1}}{\F_t}
    \ge
    \rho_t
    +
    2 \left(
      \frac{N - \rho_t}{2 N}
    \right)
    =
    \rho_t
    \left(
      1
      -
      \frac{1}{N}
    \right)
    +
    1.
  \end{dmath*}
  In general, since $\Abs{\mathbf{1}^T X} > 1$ with probability at least
  $\frac{1}{64}$,
  \begin{dmath*}
    \Exvc{\rho_{t+1}}{\F_t}
    \ge
    \left(
      1
      -
      \frac{1}{64}
    \right)
    \rho_t \left(
      1
      -
      \frac{3 - 2 \beta}{3 N}
    \right)
    +
    \frac{1}{64}
    \left(
      \rho_t
      \left(
        1
        -
        \frac{1}{N}
      \right)
      +
      1
    \right)
    =
    \rho_t \left(
      1
      -
      \left(
        1
        -
        \frac{1}{64}
      \right)
      \frac{3 - 2 \beta}{3 N}
      -
      \frac{1}{64 N}
    \right)
    +
    \frac{1}{64}
    =
    \rho_t \left(
      1
      -
      \frac{1}{N}
      +
      \left(
        1
        -
        \frac{1}{64}
      \right)
      \frac{2 \beta}{3 N}
    \right)
    +
    \frac{1}{64}
    \ge
    \rho_t \left(
      1
      -
      \frac{1}{N}
    \right)
    +
    \frac{1}{64}
  \end{dmath*}
  This expression has fixed point
  \begin{dmath*}
    \rho^*
    =
    \frac{N}{64}.
  \end{dmath*}
  Since $\rho$ is written as a sum of independent samples, as long as
  $\rho > 0$, the distribution of $\rho$ is going to be exponentially
  concentrated around its expected value, which we have just shown is
  at least $\frac{N}{64}$.  It follows that it is exponentially unlikely
  to ever achieve a value of $\rho$ that is not positive.  By the union bound,
  there is some $t = \exp(\Omega(N))$ such that, after $t$ timesteps,
  $\rho_t > 0$ with high probability.

  But, the actual probability that $\rho > 0$ in the stationary distribution is
  exactly $\frac{1}{2}$, by symmetry.  It follows that the mixing time for the
  \hogwild chain
  must be greater than $t$; that is,
  \[
    t_{\mathrm{mix-hog}} \ge \exp(\Omega(N)).
  \]
  This finishes our proof of the statement.
\end{proof}

Next, we restate and prove Theorem \ref{thmSeqHogLocalMixing}.

\thmSeqHogLocalMixing*

\begin{proof}[Proof of First Part of Theorem \ref{thmSeqHogLocalMixing}]
  If we use the coupling from Lemma \ref{lemmaSequentialLocalPB}, then
  by the result of that lemma,
  \[
    \Prob{X_{i,t} \ne Y_{i,t}}
    \le
    \exp\left(- \frac{1 - \alpha}{n} t \right),
  \]
  It follows by the union bound that
  \[
    \Prob{X_t \ne Y_t}
    \le
    n \exp\left(- \frac{1 - \alpha}{n} t \right).
  \]
  Now, assume that we initialize $X_0$ with distribution $\mu_0$, and $Y_0$
  with the stationary distribution $\pi$.  
  By Proposition \ref{propCoupling}, since $X_t$ has distribution
  $P^{(t)} \mu_0$ and
  $Y_t$ has distribution $P^{(t)} \pi$, this is equivalent to saying
  \[
    \tvdist{P^{(t)} \mu_0 - P^{(t)} \pi}
    \le
    n \exp\left(- \frac{1 - \alpha}{n} t \right).
  \]
  Therefore, in order for
  \[
    \tvdist{P^{(t)} \mu_0 - P^{(t)} \pi} \le \epsilon,
  \]
  it suffices to choose $t$ such that
  \[
    \epsilon
    =
    n \exp\left(- \frac{1 - \alpha}{n} t \right).
  \]
  This occurs when
  \[
    t
    =
    \frac{n}{1 - \alpha} \log\left( \frac{n}{\epsilon} \right),
  \]
  which is the desired expression.
\end{proof}

\begin{proof}[Proof of Second Part of Theorem \ref{thmSeqHogLocalMixing}]
  If we use the coupling from Lemma \ref{lemmaHogwildLocalPB}, then
  by the result of that lemma,
  \[
    \Prob{X_{i,t} \ne Y_{i,t}}
    \le
    \exp\left(- \frac{1 - \alpha}{n + \alpha \tau^*} t \right),
  \]
  It follows by the union bound that
  \[
    \Prob{X_t \ne Y_t}
    \le
    n \exp\left(- \frac{1 - \alpha}{n + \alpha \tau^*} t \right).
  \]
  Next, recall that we assumed that our \hogwild-Gibbs sampler
  \crcchange{has target distribution $\pi$}.  
  Now, assume that we initialize $X_0$ with distribution $\mu_0$, and $Y_0$
  with the target distribution $\pi$.  
  By Proposition \ref{propCoupling}, since $X_t$ has distribution
  $P^{(t)} \mu_0$ and $Y_t$ has distribution $P^{(t)} \pi$, this is equivalent
  to saying
  \[
    \tvdist{P^{(t)} \mu_0 - P^{(t)} \pi}
    \le
    n \exp\left(- \frac{1 - \alpha}{n + \alpha \tau^*} t \right).
  \]
  Therefore, in order for
  \[
    \tvdist{P^{(t)} \mu_0 - P^{(t)} \pi} \le \epsilon,
  \]
  it suffices to choose $t$ such that
  \[
    \epsilon
    =
    n \exp\left(- \frac{1 - \alpha}{n + \alpha \tau^*} t \right).
  \]
  This occurs when
  \[
    t
    =
    \frac{n + \alpha \tau^*}{1 - \alpha} \log\left( \frac{n}{\epsilon} \right),
  \]
  which is the desired expression.
\end{proof}

Next, we restate and prove Statement \ref{stmtExperimentalUpperBound},
which says that our experimental strategy provides a valid upper bound on
the mixing time.

\stmtExperimentalUpperBound*

\begin{proof}[Proof of Statement \ref{stmtExperimentalUpperBound}]
  Consider the partial ordering of states in this Ising model defined by
  \[
    Y \preceq X \leftrightarrows \forall i, \, Y_i \le X_i.
  \]
  Next, consider the coupling procedure that, at each time $t$, chooses
  a random variable $\tilde I_t$ to sample and a random $\tilde R_t$ uniformly
  on $[0, 1]$.  It then computes $p_t$, the marginal probability of sampling
  the chosen variable as $1$, and assigns the variable as
  \[
    \text{new value of $X_{\tilde I_t}$}
    =
    \left\{
      \begin{array}{l l}
        1 & \text{if } \tilde R_t < p_t, \\
        0 & \text{otherwise}
      \end{array}
    \right..
  \]
  This sampling procedure is equivalent to the one that we use in the
  experiment, and it will produce a chain that is consistent with the
  Ising model's dynamics.

  If we consider the evolution of two coupled chains $X^{(t)}$ and $Y^{(t)}$
  using the same values of $\tilde I_t$ and $\tilde R_t$, then from the way
  that we constructed the coupling, it follows that if
  \[
    Y^{(0)} \preceq X^{(0)},
  \]
  then for any future time $t$,
  \[
    Y^{(t)} \preceq X^{(t)}.
  \]
  This is because if
  \[
    Y^{(t)} \preceq X^{(t)},
  \]
  then the marginal probability of assigning $1$ to any particular variable in
  $X$ is always no less than the marginal probability of assigning $1$ to
  the same variable in $Y$.

  Therefore, if we initialize all $X^{(0)}_i = 1$ and all $Y^{(0)}_i = -1$,
  and run the coupling until time $T_{\text{coupling}}$, the time at which
  \[
    Y^{(T_{\text{coupling}})} = X^{(T_{\text{coupling}})},
  \]
  then by the previous analysis, since for any chain $U$ initialized at any
  state $U^{(0)}$,
  \[
    Y^{(0)} \preceq U^{(0)} \preceq X^{(0)},
  \]
  it follows that
  \[
    Y^{(T_{\text{coupling}})}
    \preceq
    U^{(T_{\text{coupling}})}
    \preceq
    X^{(T_{\text{coupling}})},
  \]
  and so,
  \[
    Y^{(T_{\text{coupling}})}
    =
    U^{(T_{\text{coupling}})}
    =
    X^{(T_{\text{coupling}})}.
  \]
  Since this was true for any initial value of $U$, it follows that
  $T_{\text{coupling}}$ is a coupling time for any two initial values of the
  chain.  Therefore, by Corollary 5.3 from \citetsec{levin2009markov},
  \[
    \max_{\mu_0} \tvdist{P^{(t)} \mu_0 - \pi}
    \le
    \Prob{T_{\text{coupling}} > t}.
  \]
  If we use our definition of $\hat t(\epsilon)$ where
  \[
    \Prob{T_{\text{coupling}} > \hat t(\epsilon)} = \epsilon,
  \]
  then this implies that
  \[
    \max_{\mu_0} \tvdist{P^{(\hat t)} \mu_0 - \pi}
    \le
    \epsilon.
  \]
  This in turn implies that $\hat t$ is a upper bound on the mixing time,
  which is the desired result.
\end{proof}

\subsection{Proofs of Lemmas}

In this section, we will restate and prove the lemmas used earlier in the
appendix.

\lemmaSparseCoupling*

\begin{proof}[Proof of Lemma \ref{lemmaSparseCoupling}]
  For any set of variables $I \subset \{1,\ldots,n\}$, let $M_I(\mu)$ denote
  the marginal distribution of the variables in $I$ in the distribution $\mu$.
  In particular, $M_I$ includes all events $A$ that depend only on variables
  in set $I$.  Next, let $\bar X_I$ and $\bar Y_I$ denote the values of
  $\bar X$ and $\bar Y$ on those variables in $I$; this will be a coupling of
  the distributions $M_I(\mu)$ and $M_I(\nu)$. Therefore, by Proposition
  \ref{propCoupling},
  \[
    \svdist{M_I(\mu) - M_I(\nu)}{\omega}
    \le
    \Prob{\bar X_I \ne \bar Y_I}
    =
    \Prob{\exists i \in I, \: \bar X_i \ne \bar Y_i}.
  \]
  Let $\Omega_I$ denote all events in the original probability space $\Omega$
  that depend only on the variables in $I$.
  By the definition of total variation distance,
  \[
    \svdist{M_I(\mu) - M_I(\nu)}{\omega}
    =
    \max_{A \in \Omega_I}
    \Abs{\mu(A) - \nu(A)}.
  \]
  Therefore,
  \[
    \max_{A \in \Omega_I}
    \Abs{\mu(A) - \nu(A)}
    \le
    \Prob{\exists i \in I, \: \bar X_i \ne \bar Y_i}.
  \]
  Now, since this was true for any $I$, it is certainly true if we maximize
  both sides over all $I$ with $\Abs{I} \le \omega$.  Therefore,
  \[
    \max_{I \subseteq \{1,\ldots,n\}, \, \Abs{I} \le \omega}
    \max_{A \in \Omega_I}
    \Abs{\mu(A) - \nu(A)}
    \le
    \max_{I \subseteq \{1,\ldots,n\}, \, \Abs{I} \le \omega}
    \Prob{\exists i \in I, \: \bar X_i \ne \bar Y_i}.
  \]
  The left side can be reduced to
  \[
    \max_{\Abs{A} \le \omega}
    \Abs{\mu(A) - \nu(A)}
    \le
    \max_{I \subseteq \{1,\ldots,n\}, \, \Abs{I} \le \omega}
    \Prob{\exists i \in I, \: \bar X_i \ne \bar Y_i}
  \]
  and applying the definition of sparse variation distance proves the lemma.
\end{proof}

\lemmaLocalInfExvDist*

\begin{proof}[Proof of Lemma \ref{lemmaLocalInfExvDist}]
  Let $n$ be the number of variables in the model.  For all
  $k \in \{0, 1, \ldots, n \}$, let $Z_k$ be a random variable that takes
  on values in the state space of $\pi$ such that, for all
  $j \in \{1, \ldots, n\}$,
  \[
    Z_{k,j}
    =
    \left\{
      \begin{array}{r l}
        X_j & \text{ if } j > k \\
        Y_j & \text{ if } j \le k
      \end{array}
    \right..
  \]
  In particular, $Z_0 = X$ and $Z_n = Y$.  Now, by the triangle inequality
  on the total variation distance,
  \begin{align*}
    \tvdist{
      \pi_i(\cdot | X)
      -
      \pi_i(\cdot | Y)
    }
    &=
    \tvdist{
      \pi_i(\cdot | Z_0)
      -
      \pi_i(\cdot | Z_n)
    } \\
    &\le
    \sum_{k=1}^n
    \tvdist{
      \pi_i(\cdot | Z_{k-1})
      -
      \pi_i(\cdot | Z_k)
    }
  \end{align*}
  Next, we note that $Z_{k-1} = Z_k$ if and only if $X_k = Y_k$.  Therefore, 
  \[
    \tvdist{
      \pi_i(\cdot | X)
      -
      \pi_i(\cdot | Y)
    }
    \le
    \sum_{k=1}^n
    \mathbf{1}_{X_k \ne Y_k}
    \tvdist{
      \pi_i(\cdot | Z_{k-1})
      -
      \pi_i(\cdot | Z_k)
    }.
  \]
  Since $Z_{k-1}$ and $Z_k$ differ only at most at index $k$, it follows
  that $(Z_{k-1}, Z_k) \in B_k$, and so,
  \[
    \tvdist{
      \pi_i(\cdot | X)
      -
      \pi_i(\cdot | Y)
    }
    \le
    \sum_{k=1}^n
    \mathbf{1}_{X_k \ne Y_k}
    \max_{(U,V) \in B_k}
    \tvdist{
      \pi_i(\cdot | U)
      -
      \pi_i(\cdot | V)
    }.
  \]
  Maximizing over the right side produces
  \[
    \tvdist{
      \pi_i(\cdot | X)
      -
      \pi_i(\cdot | Y)
    }
    \le
    \max_j
    \sum_{k=1}^n
    \mathbf{1}_{X_k \ne Y_k}
    \max_{(U,V) \in B_k}
    \tvdist{
      \pi_j(\cdot | U)
      -
      \pi_j(\cdot | V)
    }.
  \]
  Taking the expected value of both sides produces
  \begin{align*}
    \Exv{
      \tvdist{
        \pi_i(\cdot | X)
        -
        \pi_i(\cdot | Y)
      }
    }
    &\le
    \max_j
    \sum_{k=1}^n
    \Exv{
      \mathbf{1}_{X_k \ne Y_k}
    }
    \max_{(U,V) \in B_k}
    \tvdist{
      \pi_j(\cdot | U)
      -
      \pi_j(\cdot | V)
    } \\
    &=
    \max_j
    \sum_{k=1}^n
    \Prob{X_k \ne Y_k}
    \max_{(U,V) \in B_k}
    \tvdist{
      \pi_j(\cdot | U)
      -
      \pi_j(\cdot | V)
    } \\
    &\le
    \left(
      \max_k 
      \Prob{X_k \ne Y_k}
    \right)
    \left(
      \max_j
      \sum_{k=1}^n
      \max_{(U,V) \in B_k}
      \tvdist{
        \pi_j(\cdot | U)
        -
        \pi_j(\cdot | V)
      }
    \right).
  \end{align*}
  Finally, applying the definition of total influence gives us
  \[
    \Exv{
      \tvdist{
        \pi_i(\cdot | X)
        -
        \pi_i(\cdot | Y)
      }
    }
    \le
    \alpha
    \max_k 
    \Prob{X_k \ne Y_k}.
  \]
  This proves the lemma.
\end{proof}

\lemmaSequentialLocalPB*

\begin{proof}[Proof of Lemma \ref{lemmaSequentialLocalPB}]
  Define the coupling as follows.  Start in state $(X_0, Y_0)$, and at each
  timestep, choose a single variable $i$ uniformly at random for both chains to
  sample. Then, sample the selected variable in both chains using the
  optimal coupling, \crcchange{of the conditional distributions of the variable
  to be sampled in both chains,}
  guaranteed by Proposition \ref{propCoupling}.  Iterated over time, this 
  defines a full coupling of the two chains.

  Next, consider the event that $X_{i,t+1} \ne Y_{i,t+1}$.  This event will
  occur if one of two things happens: either we didn't sample variable $i$ at
  time $t$ and $X_{i,t} \ne Y_{i,t}$; or we did sample variable $i$ at time
  $t$, and the sampled variables were not equal.  Since the probability of
  sampling variable $i$ is $\frac{1}{n}$, and we know the probability that
  the sampled variables were not equal from Proposition \ref{propCoupling},
  it follows that, by the law of total probability,
  \begin{dmath*}
    \Prob{X_{i,t+1} \ne Y_{i,t+1}}
    =
    \left( 1 - \frac{1}{n} \right) \Prob{X_{i,t} \ne Y_{i,t}}
    +
    \frac{1}{n}
    \Exv{
      \tvdist{
        \pi_i(\cdot | X_t)
        -
        \pi_i(\cdot | Y_t)
      }
    },
  \end{dmath*}
  where $\pi_i(\cdot | X_t)$ denotes the conditional distribution of variable
  $i$ in $\pi$ given the values of the other variables in $X_t$.

  Next, we apply the Lemma \ref{lemmaLocalInfExvDist}, which gives us
  \begin{dmath*}
    \Prob{X_{i,t+1} \ne Y_{i,t+1}}
    \le
    \left( 1 - \frac{1}{n} \right) \Prob{X_{i,t} \ne Y_{i,t}}
    +
    \frac{\alpha}{n}
    \max_j \Prob{X_{j,t} \ne Y_{j,t}}.
  \end{dmath*}
  Maximizing both sides over $i$ produces
  \begin{dmath*}
    \max_i \Prob{X_{i,t+1} \ne Y_{i,t+1}}
    \le
    \left( 1 - \frac{1}{n} \right) \max_i \Prob{X_{i,t} \ne Y_{i,t}}
    +
    \frac{\alpha}{n}
    \max_j \Prob{X_{j,t} \ne Y_{j,t}}
    =
    \left( 1 - \frac{1}{n} + \frac{\alpha}{n} \right)
    \max_i \Prob{X_{i,t} \ne Y_{i,t}}.
  \end{dmath*}
  Applying this inequality recursively, and noting that
  $\max_i \Prob{X_{i,0} \ne Y_{i,0}} \le 1$, we get
  \[
    \max_i \Prob{X_{i,t} \ne Y_{i,t}}
    \le
    \left( 1 - \frac{1 - \alpha}{n} \right)^t
    \le
    \exp\left( - \frac{1 - \alpha}{n} t \right).
  \]
  This gives us the desired result.
\end{proof}

\lemmaHogwildLocalPB*

\begin{proof}[Proof of Lemma \ref{lemmaHogwildLocalPB}]
  Define the coupling as follows.  Start in state $(X_0, Y_0)$, and at each
  timestep, choose a single variable $i$ uniformly at random for both chains to
  sample.  Similarly, choose the \hogwild delays $\tilde \tau_{i,t}$ to also be
  the same between the two chains.  At time $t$, let $\tilde U_t$ denote the
  state that would be read by chain $X$'s sampler based on the delays, and 
  similarly let $\tilde V_t$ denote the state that would be read by chain $Y$'s
  sampler.  That is,
  \[
    \tilde U_{i,t} = X_{i,t-\tilde \tau_{i,t}},
  \]
  and similarly,
  \[
    \tilde V_{i,t} = Y_{i,t-\tilde \tau_{i,t}}.
  \]
  As in the sequential case, we sample the selected variable in both chains
  using the optimal coupling \crcchange{(of the conditional distributions of the
  variable to be sampled in both chains)} guaranteed
  by Proposition \ref{propCoupling}.
  Iterated over time, this defines a full coupling of the two chains.

  We follow the same argument as in the sequential case.  First,
  consider the event that $X_{i,t+1} \ne Y_{i,t+1}$.  This event will
  occur if one of two things happens: either we didn't sample variable $i$ at
  time $t$ and $X_{i,t} \ne Y_{i,t}$; or we did sample variable $i$ at time
  $t$, and the sampled variables were not equal.  Since the probability of
  sampling variable $i$ is $\frac{1}{n}$, and we know the probability that
  the sampled variables were not equal from Proposition \ref{propCoupling},
  it follows that, by the law of total probability,
  \begin{dmath*}
    \Prob{X_{i,t+1} \ne Y_{i,t+1}}
    =
    \left( 1 - \frac{1}{n} \right) \Prob{X_{i,t} \ne Y_{i,t}}
    +
    \frac{1}{n}
    \Exv{
      \tvdist{
        \pi_i(\cdot | \tilde U_t)
        -
        \pi_i(\cdot | \tilde V_t)
      }
    },
  \end{dmath*}
  where $\pi_i(\cdot | X_t)$ denotes the conditional distribution of variable
  $i$ in $\pi$ given the values of the other variables in $X_t$.

  Next, we apply the Lemma \ref{lemmaLocalInfExvDist}, which gives us
  \begin{dmath*}
    \Prob{X_{i,t+1} \ne Y_{i,t+1}}
    \le
    \left( 1 - \frac{1}{n} \right) \Prob{X_{i,t} \ne Y_{i,t}}
    +
    \frac{\alpha}{n}
    \max_j \Prob{U_{j,t} \ne V_{j,t}}
    =
    \left( 1 - \frac{1}{n} \right) \Prob{X_{i,t} \ne Y_{i,t}}
    +
    \frac{\alpha}{n}
    \max_j
    \sum_{k = 0}^{\infty}
    \Prob{\tilde \tau_{j,t} = k}
    \Prob{X_{j,t - k} \ne Y_{j,t - k}}.
  \end{dmath*}
  Now, if we let
  \[
    \phi_t
    =
    \max_i \Prob{X_{i,t} \ne Y_{i,t}},
  \]
  then maximizing the previous expression over $i$ implies that
  \begin{dmath*}
    \phi_{t+1}
    \le
    \left( 1 - \frac{1}{n} \right) \phi_t
    +
    \frac{\alpha}{n}
    \max_j
    \sum_{k = 0}^{\infty}
    \Prob{\tilde \tau_{j,t} = k}
    \phi_{t-k}.
  \end{dmath*}
  Now, for some constant $r \le n^{-1}$, let $y_t$ be defined to be the
  sequence
  \[
    y_t = \exp(-rt).
  \]
  Then, notice that
  \begin{dmath*}
    \left( 1 - \frac{1}{n} \right) y_t
    +
    \frac{\alpha}{n}
    \max_j
    \sum_{k = 0}^{\infty}
    \Prob{\tilde \tau_{j,t} = k}
    y_{t-k}
    =
    \left( 1 - \frac{1}{n} \right) \exp(-rt)
    +
    \frac{\alpha}{n}
    \max_j
    \sum_{k = 0}^{\infty}
    \Prob{\tilde \tau_{j,t} = k}
    \exp(-rt+rk)
    =
    \exp(-rt) \left(
      \left( 1 - \frac{1}{n} \right)
      +
      \frac{\alpha}{n}
      \max_j
      \sum_{k = 0}^{\infty}
      \Prob{\tilde \tau_{j,t} = k}
      \exp(rk)
    \right)
    =
    \exp(-rt) \left(
      \left( 1 - \frac{1}{n} \right)
      +
      \frac{\alpha}{n}
      \max_j
      \Exv{\exp(r \tilde \tau_{j,t})}
    \right).
  \end{dmath*}
  Now, by the convexity of the exponential function,
  \begin{dmath*}
    \left( 1 - \frac{1}{n} \right) y_t
    +
    \frac{\alpha}{n}
    \max_j
    \sum_{k = 0}^{\infty}
    \Prob{\tilde \tau_{j,t} = k}
    y_{t-k}
    \le
    \exp(-rt) \left(
      \left( 1 - \frac{1}{n} \right)
      +
      \frac{\alpha}{n}
      \max_j
      \left(
        1
        +
        r n
        \Exv{\exp\left( \frac{\tilde \tau_{j,t}}{n} \right) - 1}
      \right)
    \right).
  \end{dmath*}
  Applying the constraint that
  \[
    \Exv{\exp\left( \frac{\tilde \tau_{j,t}}{n} \right)}
    \le
    1 + \frac{\tau^*}{n},
  \]
  we can reduce this to
  \begin{dmath*}
    \left( 1 - \frac{1}{n} \right) y_t
    +
    \frac{\alpha}{n}
    \max_j
    \sum_{k = 0}^{\infty}
    \Prob{\tilde \tau_{j,t} = k}
    y_{t-k}
    \le
    \exp(-rt) \left(
      \left( 1 - \frac{1}{n} \right)
      +
      \frac{\alpha}{n}
      \left(
        1
        +
        r \tau^*
      \right)
    \right)
    =
    y_{t+1} \exp(r) \left(
      1 
      -
      \frac{1}{n}
      +
      \frac{\alpha}{n}
      +
      \frac{r \alpha \tau^*}{n}
    \right)
    \le
    y_{t+1} \exp(r) \exp\left(
      -
      \frac{1}{n}
      +
      \frac{\alpha}{n}
      +
      \frac{r \alpha \tau^*}{n}
    \right)
    =
    y_{t+1} \exp\left(
      \frac{n + \alpha \tau^*}{n} r
      -
      \frac{1 - \alpha}{n}
    \right).
  \end{dmath*}
  Now, we choose $r$ such that the argument to this exponential is zero; that
  is, we choose
  \[
    r
    =
    \frac{1 - \alpha}{n + \alpha \tau^*}.
  \]
  Notice that this choice satisfies the earlier assumption that
  $0 < r \le n^{-1}$.  Using this choice, we can conclude that
  \begin{dmath*}
    y_{t+1}
    \ge
    \left( 1 - \frac{1}{n} \right) y_t
    +
    \frac{\alpha}{n}
    \max_j
    \sum_{k = 0}^{\infty}
    \Prob{\tilde \tau_{j,t} = k}
    y_{t-k}.
  \end{dmath*}
  Therefore, by Lemma \ref{lemmaMSDL},
  \[
    \phi_t
    \le
    y_t
    =
    \exp\left(- \frac{1 - \alpha}{n + \alpha \tau^*} t \right).
  \]
  This proves the lemma.
\end{proof}

\lemmaHogwildLocalBias*

\begin{proof}[Proof of Lemma \ref{lemmaHogwildLocalBias}]
  Define the coupling as follows.  Start in state $(X_0, Y_0)$, and at each
  timestep, choose a single variable $\tilde I_t$ uniformly at random for both
  chains to
  sample.  Then, choose the delays $\tilde \tau_{i,t}$ for the \hogwild chain
  $X_t$.  At time $t$, let $\tilde U_t$ denote the
  state that would be read by chain $X$'s sampler based on the delays.
  That is,
  \[
    \tilde U_{i,t} = X_{i,t-\tilde \tau_{i,t}}.
  \]
  As done previously, we sample the selected variable $\tilde I_t$ in both
  chains using the optimal coupling guaranteed by Proposition
  \ref{propCoupling}.
  Iterated over time, this defines a full coupling of the two chains.

  We follow a similar argument as in the above lemmas used to bound the mixing
  time.  First,
  consider the event that $X_{i,t+1} \ne Y_{i,t+1}$.  This event will
  occur if one of two things happens: either we didn't sample variable $i$ at
  time $t$ and $X_{i,t} \ne Y_{i,t}$; or we did sample variable $i$ at time
  $t$, and the sampled variables were not equal.  Since the probability of
  sampling variable $i$ is $\frac{1}{n}$, and we know the probability that
  the sampled variables were not equal from Proposition \ref{propCoupling},
  it follows that, by the law of total probability,
  \begin{dmath*}
    \Prob{X_{i,t+1} \ne Y_{i,t+1}}
    =
    \left( 1 - \frac{1}{n} \right) \Prob{X_{i,t} \ne Y_{i,t}}
    +
    \frac{1}{n}
    \Exv{
      \tvdist{
        \pi_i(\cdot | \tilde U_t)
        -
        \pi_i(\cdot | \tilde Y_t)
      }
    },
  \end{dmath*}
  where $\pi_i(\cdot | X_t)$ denotes the conditional distribution of variable
  $i$ in $\pi$ given the values of the other variables in $X_t$.

  Next, we apply the Lemma \ref{lemmaLocalInfExvDist}, which gives us
  \begin{dmath*}
    \Prob{X_{i,t+1} \ne Y_{i,t+1}}
    \le
    \left( 1 - \frac{1}{n} \right) \Prob{X_{i,t} \ne Y_{i,t}}
    +
    \frac{\alpha}{n}
    \max_j \Prob{U_{j,t} \ne Y_{j,t}}
    =
    \left( 1 - \frac{1}{n} \right) \Prob{X_{i,t} \ne Y_{i,t}}
    +
    \frac{\alpha}{n}
    \max_j
    \sum_{k = 0}^{\infty}
    \Prob{\tilde \tau_{j,t} = k}
    \Prob{X_{j,t - k} \ne Y_{j,t}}.
  \end{dmath*}
  In order to evaluate this, we notice that the event $X_{j,t - k} \ne Y_{j,t}$
  can happen only if either $X_{j,t} \ne Y_{j,t}$ or at some time $s$, where
  $t - k \le s < t$, we sampled variable $j$ (that is, $\tilde I_s = j$). 
  Therefore, by the union bound,
  \[
    \Prob{X_{j,t - k} \ne Y_{j,t}}
    \le
    \Prob{X_{j,t} \ne Y_{j,t}}
    +
    \sum_{s = t - k}^{t - 1} \Prob{\tilde I_s = j}.
  \]
  Since the probability of sampling variable $j$ at any time is always just
  $\frac{1}{n}$, we can reduce this to
  \[
    \Prob{X_{j,t - k} \ne Y_{j,t}}
    \le
    \Prob{X_{j,t} \ne Y_{j,t}}
    +
    \frac{k}{n}.
  \]
  Substituting this into our previous expression produces
  \begin{dmath*}
    \Prob{X_{i,t+1} \ne Y_{i,t+1}}
    \le
    \left( 1 - \frac{1}{n} \right) \Prob{X_{i,t} \ne Y_{i,t}}
    +
    \frac{\alpha}{n}
    \max_j
    \sum_{k = 0}^{\infty}
    \Prob{\tilde \tau_{j,t} = k} \left(
      \Prob{X_{j,t} \ne Y_{j,t}}
      +
      \frac{k}{n}
    \right)
    =
    \left( 1 - \frac{1 - \alpha}{n} \right) \Prob{X_{i,t} \ne Y_{i,t}}
    +
    \frac{\alpha}{n^2}
    \max_j
    \Exv{\tilde \tau_{j,t}}
    \le
    \left( 1 - \frac{1 - \alpha}{n} \right) \Prob{X_{i,t} \ne Y_{i,t}}
    +
    \frac{\alpha \tau}{n^2}.
  \end{dmath*}
  Now, if we let
  \[
    \phi_t
    =
    \max_i \Prob{X_{i,t} \ne Y_{i,t}},
  \]
  then maximizing the previous expression over $i$ implies that
  \[
    \phi_{t+1}
    \le
    \left( 1 - \frac{1 - \alpha}{n} \right) \phi_t
    +
    \frac{\alpha \tau}{n^2}.
  \]
  Subtracting from both sides to identify the fixed point gives us
  \begin{dmath*}
    \phi_{t+1}
    -
    \frac{\alpha \tau}{(1 - \alpha) n}
    \le
    \left( 1 - \frac{1 - \alpha}{n} \right) \phi_t
    +
    \frac{\alpha \tau}{n^2}
    - 
    \frac{\alpha \tau}{(1 - \alpha) n}
    =
    \left( 1 - \frac{1 - \alpha}{n} \right) \left(
      \phi_t
      -
      \frac{\alpha \tau}{(1 - \alpha) n}
    \right).
  \end{dmath*}
  Applying this inequality recursively lets us conclude that
  \begin{dmath*}
    \phi_t
    -
    \frac{\alpha \tau}{(1 - \alpha) n}
    \le
    \left( 1 - \frac{1 - \alpha}{n} \right)^t \left(
      \phi_0
      -
      \frac{\alpha \tau}{(1 - \alpha) n}
    \right)
    \le
    \exp\left( - \frac{1 - \alpha}{n} t \right),
  \end{dmath*}
  and so,
  \[
    \phi_t
    \le
    \exp\left( - \frac{1 - \alpha}{n} t \right)
    +
    \frac{\alpha \tau}{(1 - \alpha) n}.
  \]
  This is the desired expression.
\end{proof}

\lemmaMSDL*

\begin{proof}[Proof of Lemma \ref{lemmaMSDL}]
  We will approach this by induction.  The base case holds by assumption, since
  $x_0 = y_0$.  For the inductive case, if $x_t \le y_t$ for all $t \le T$, then
  \[
    x_{T+1}
    \le
    f_T(x_T, x_{T-1}, \ldots, x_0).
  \]
  By monotonicity and the inductive hypothesis,
  \[
    x_{T+1}
    \le
    f_T(y_T, y_{T-1}, \ldots, y_0),
  \]
  and therefore,
  \[
    x_{T+1} \le y_{T+1}.
  \]
  Applying induction to this proves the lemma.
\end{proof}

\lemmaModelX*

\begin{proof}[Proof of Lemma \ref{lemmaModelX}]
  \crcchange{(This lemma contains much of the technical work needed to prove
  Statement \ref{stmtBadMixExample}.  A higher-level motivation for why
  we are proving this lemma is furnished in the proof of that result.)}

  Assume that, as we run the chain described in this lemma, we also assign
  a ``color'' to each of the variables.  All variables with an initial value
  of $1$ start out as black, and all other variables start out as white.  Let
  $B_t$ denote the set of variables that are colored black at any time $t$, and
  let $S_t$ denote the sum of all variables that are colored black at that
  time.  We re-color variables according to the following procedure:
  \begin{enumerate}
    \item Whenever we change a variable's value from $-1$ to $1$, if it is
      colored white, color it black.
    \item Whenever we change a variable's value from $-1$ to $1$, if it is
      already colored black, choose a random variable that had value $-1$ at
      time $t$, and if it is white, color it black.
  \end{enumerate}
  Note that as a consequence of this result, a variable that is colored
  white always has value $-1$.

  We will prove the following sub-result by induction on $t$: given a time $t$,
  set $B_t$, and sum $S_t$, the values of the variables in $B_t$ are uniformly
  distributed over the set of possible assignments that are consistent with
  $S_t$.

  (Base Case.)  The base case is straightforward.  Since $B_0$ is just the set
  of variables that have value $1$, there is only one possible assignment
  that is consistent with $S_0$: the assignment in which all variables take
  on the value $1$.  Since this assignment actually occurs with probability
  $1$, the statement holds.

  (Inductive Case.)  Assume that the sub-result is true at time $t$.  The
  sampler chooses a new variable $i$ to sample.  One of the following things
  will happen:
  \begin{itemize}
    \item We don't re-color any variables, or change the values of any variables
      in $B_t$.  In this case, $B_{t+1} = B_t$ and $S_{t+1} = S_t$.  Since
      there is no change to $B$ or $S$, all consistent
      assignments of the black variables are still equiprobable.
    \item We don't re-color any variables, but we do change the value of some
      variable in $B_t$ (by changing its value from $1$ to $-1$).  Since
      we sampled the variable $i$ at random, all consistent assignments of
      the black variables will remain equiprobable.
    \item We re-color some variable $j$ black.  There are two events that can
      cause this:
      \begin{itemize}
        \item We could have sampled variable $j$ (that is $i = j$), and changed
          its value from $-1$ to $1$.  This will happen with probability
          \[
            \frac{1}{N} \cdot \frac{1}{2} = \frac{1}{2 N}
          \]
        \item We could have sampled a variable $i \ne j$ that is already colored
          black, changed its value from $-1$ to $1$, and then chosen variable
          $j$ at random to color black.  Since, at time $t$, the number of
          variables with value $-1$ must be
          \[
            \frac{N + 1}{2},
          \]
          (since we are about to change a value from $-1$ to $1$),
          this will happen with probability
          \[
            \frac{u}{N} \cdot \frac{1}{2} \cdot \frac{2}{N + 1}
            =
            \frac{u}{N (N + 1)}
          \]
          where $u$ is the number of black-colored variables that have value
          $-1$ at time $t$.
      \end{itemize}
      From this analysis, it follows that, given that we re-colored some
      variable $j$ black, it will have value $-1$ with probability
      \[
        \Prob{\text{variable $j$ has value $-1$}}
        =
        \frac{\frac{u}{N (N + 1)}}{\frac{1}{2 N} + \frac{u}{N (N + 1)}}
        =
        \frac{u}{u + \frac{N + 1}{2}}.
      \]
      In particular, at time $t$, the number of variables that are in $B_t$ is
      \[
        \frac{N - 1}{2} + u,
      \]
      since all variables with value $1$ are in $B_t$, and $B_t$ is stipulated
      to contain $u$ additional variables with value $-1$.  It follows that
      at time $t+1$, the number of variables that are in $B_t$ is
      \[
        \frac{N + 1}{2} + u,
      \]
      and there will still be $u$ variables in $B_{t+1}$ with value $-1$.
      Therefore, the fraction of variables in $B_{t+1}$ that have value $-1$
      will be
      \[
        \frac{u}{u + \frac{N + 1}{2}}.
      \]
      Note that this is exactly equal to the probability that variable $j$
      will have value $-1$.  Combining this with the inductive hypothesis
      shows that the consistent states will all remain equiprobable in this
      case.
  \end{itemize}
  Since the consistent states remain equiprobable in all of the possible cases,
  it follows from the law of total probability that the consistent states are
  equiprobable in all cases.  This shows that the sub-result holds in the
  inductive case.

  We have now showed that given a time $t$,
  set $B_t$, and sum $S_t$, the values of the variables in $B_t$ are uniformly
  distributed over the set of possible assignments that are consistent with
  $S_t$.  This implies that if $T_1$ is the first time at which the set $B_t$
  contains all
  variables, the value of $X_T$ is are uniformly distributed
  over all possible states with $\mathbf{1}^T X = 1$.

  Now, we performed this construction for a particular polarity of swaps
  (i.e. focusing on switches from $-1$ to $1$), but by symmetry we could
  just as easily have used the same construction with the signs of all the
  variables reversed.  If we let $T_{-1}$ be the first time at which the
  set $B_t$ contains all variables using this reverse-polarity construction,
  then the value of $X_T$ is uniformly distributed
  over all possible states with $\mathbf{1}^T X = -1$.

  Let $T^*$ be a random variable that is $T_1$ with probability $\frac{1}{2}$
  and $T_{-1}$ with probability $\frac{1}{2}$.  It follows that at time $T^*$,
  the distribution of $X_{T^*}$ will be $\pi$.  Therefore, $T^*$ is a strong
  stationary time for this chain.  By the properties of strong stationary
  times,
  \[
    t_{\mathrm{mix}} \le 4 \Exv{T^*}.
  \]

  To bound the mixing time, we start by noticing that
  \[
    \Exv{T^*}
    =
    \frac{1}{2} \Exv{T_1} + \frac{1}{2} \Exv{T_{-1}}
    =
    \Exv{T_1}.
  \]
  If we let $\bar T$ be the first time at which each variable has been
  set to $1$ at least once, then
  \[
    T_1 \le \bar T.
  \]
  Now, if we sample a variable, the probability that we will set it to $1$
  is (roughly) $\frac{1}{4}$.  It follows from the coupon collector's problem
  bound that the expected amount of time required to set all variables to $1$
  at least once is
  \[
    \Exv{\bar T} = O(n \log n).
  \]
  Combining this with the previous inequalities lets us conclude that
  \[
    t_{\mathrm{mix}} = O(n \log n),
  \]
  which proves the lemma.
\end{proof}

\bibliographysec{references}
\bibliographystylesec{icml2015}

}

\end{document}